\documentclass[11pt,onecolumn]{article}

\usepackage{fullpage}
\usepackage{graphicx}
\usepackage{amsthm}
\usepackage{mathrsfs}
\usepackage{amssymb}
\usepackage{amsmath}
\usepackage{enumitem}
\usepackage{verbatim}
\usepackage{authblk}
\usepackage{array}
\usepackage[authoryear,round]{natbib}
\usepackage{setspace}
\usepackage{color}
\usepackage[plainpages=false]{hyperref}
\usepackage{mathtools}
\usepackage{dsfont}
\usepackage{sectsty}
\usepackage{tikz}
\usepackage{caption}
\usepackage{subcaption}
\usepackage{pgfplots}
\usepackage{algorithm}
\usepackage{algpseudocode}

\usepackage{cleveref}
\usepackage{booktabs}

\usepackage{rotating}
\usepackage{multirow}
\usepackage[left=1in,right=1in,top=1in,bottom=1in,footnotesep=0.5cm]{geometry}
\usepackage{bm}
\usepackage{bbm}

\usepackage{color,soul}
\setstcolor{newcc}

\usepackage{csquotes}


\usepackage{etoolbox}
\let\bbordermatrix\bordermatrix
\patchcmd{\bbordermatrix}{8.75}{4.75}{}{}
\patchcmd{\bbordermatrix}{\left(}{\left[}{}{}
\patchcmd{\bbordermatrix}{\right)}{\right]}{}{}

\setul{0.5ex}{0.3ex}
\setulcolor{blue}
\setstcolor{blue}

\newtheoremstyle{theoremsansserif} 
    {\topsep}                    
    {\topsep}                    
    {\itshape}                   
    {}                           
    {\sffamily\bfseries }        
    {.}                          
    {.5em}                       
    {}  

\theoremstyle{theoremsansserif}
\newtheorem{thrm}{Theorem}[section]
\newtheorem{theorem}{Theorem}[section]
\newtheorem{lemma}[thrm]{Lemma}
\newtheorem{lem}[thrm]{Lemma}
\newtheorem{prop}[thrm]{Proposition}
\newtheorem{corollary}[thrm]{Corollary}
\newtheorem{conj}[thrm]{Conjecture}

\theoremstyle{definition}

\newtheorem{remark}{\sffamily\bfseries Remark}[section]

\newtheorem{assumption}{\sffamily\bfseries Assumption}[section]

\newcommand{\E}{{\mathbb{E}}}

\newcommand{\RR}{\mathbb{R}}

\newcommand{\bn}{\mathbf{n}} 
\newcommand{\bN}{\mathbf{N}} 

\newcommand{\bmu}{\boldsymbol{\mu}}

\newcommand{\cP}{\mathcal{P}}

\newcommand{\beqn}{\begin{eqnarray*}}
\newcommand{\eeqn}{\end{eqnarray*}}

\newcommand{\beq}{\begin{eqnarray}}
\newcommand{\eeq}{\end{eqnarray}}

\newcommand{\PP}{{\mathbb{P}}}

\definecolor{color1}{rgb}{0,0.5,0}

\allowdisplaybreaks

\allsectionsfont{\sffamily}
\sffamily

\DeclareSymbolFont{extraup}{U}{zavm}{m}{n}
\DeclareMathSymbol{\varheart}{\mathalpha}{extraup}{86}
\DeclareMathSymbol{\vardiamond}{\mathalpha}{extraup}{87}

\definecolor{darkblue}{rgb}{0.05,0.25,0.60}
\definecolor{darkgreen}{rgb}{0,0.6,0}
\definecolor{darkred}{rgb}{0.75,0,0}
\definecolor{dbbpurple}{rgb}{0.5,0,0.9}

\def\NNN{\nonumber}%

\newcommand{\nat}{n^{\star}_{1, T}}
\newcommand{\nbt}{n^{\star}_{2, T}}
\newcommand{\wa}{\omega^{\epsilon}_{T}}
\newcommand{\wb}{\omega^{-\epsilon}_{T}}
\newcommand{\bamu}{\bar\mu}
\newcommand{\nad}{n^\delta_{1,T}}
\newcommand{\nbd}{n^\delta_{2,T}}
\newcommand{\lx}{\lambda^{\star}}

\pgfplotsset{compat=1.17}
\crefdefaultlabelformat{#2#1#3}

\begin{document}
\setstretch{1.3}
\doublespace


\newcommand{\yc}[1]{{\color{purple}[YC: #1]}}
\newcommand{\jl}[1]{{\color{violet}[JL: #1]}}
\newcommand{\ycedit}[1]{{\color{teal}#1}}

\title{\sf{\fontsize{16.0pt}{18pt}\textbf{A characterization of sample adaptivity in UCB data}}\\
}
\author[$\dag$]{\normalsize Yilun Chen}
\author[$\dag \ddag$]{\normalsize Jiaqi Lu}

\affil[$\dag$]{\sf School of Data Science, the Chinese University of Hong Kong, Shenzhen (CUHK Shenzhen)}
\affil[$\ddag$]{\sf School of Management and Economics, the Chinese University of Hong Kong, Shenzhen (CUHK Shenzhen)}
\affil[ ]{\texttt{chenyilun@cuhk.edu.cn, lujiaqi@cuhk.edu.cn}}

\date{\sf \today}	

{\singlespace \maketitle}

\begin{abstract}\singlespace
We characterize a joint CLT of the number of pulls and the sample mean reward of the arms in a stochastic two-armed bandit environment under UCB algorithms. Several implications of this result are in place: (1) a nonstandard CLT of the number of pulls hence pseudo-regret that smoothly interpolates between a standard form in the large arm gap regime and a slow-concentration form in the small arm gap regime, and (2) a heuristic derivation of the sample bias up to its leading order from the correlation between the number of pulls and sample means. Our analysis framework is based on a novel perturbation analysis, which is of broader interest on its own. 

\vspace{10mm}
\noindent\it{Key Words}: \rm  multi-armed bandit, UCB, sample adaptivity, joint CLT, slow concentration, pseudo-regret, sample bias
\end{abstract}

\thispagestyle{empty}
\pagebreak

\setcounter{page}{1}

\setstretch{1.46}

\pagebreak



\section{Introduction}
Multi-armed bandit (MAB) is a classic problem in reinforcement learning and decision theory with both theoretical appeal and practical relevance. A vast majority of the extensive MAB literature focus on designing algorithms and establishing their (expected) regret performance guarantees (\cite{lai1985asymptotically}, \cite{Auer2002}, \cite{agrawal2012analysis}, \cite{garivier2011kl}, \cite{kaufmann2012bayesian}, etc.). On the applied side, the MAB model is often considered in the design of clinical trials (\cite{thall2007practical}, \cite{press2009bandit}, \cite{magirr2012generalized}, \cite{villar2015multi}), {pricing experiments (\cite{misra2019dynamic}, \cite{calvano2020artificial}, \cite{wang2021multimodal}), portfolio selection (\cite{gagliolo2011algorithm}, \cite{shen2015portfolio}, \cite{huo2017risk}), content recommendation (\cite{li2010contextual}), and others. The unprecedented proliferation of learning algorithms and adaptive experiments across diverse applications is becoming an unignorable data source.} 
This motivates a recent surge in people's interest towards a better statistical understanding of bandit data, {potentially applicable to performing downstream inference tasks}. For example, \cite{kalvit2021closer} and \cite{fan2022typical} establish LLN and CLT of pseudo-regret (a classic notion of bandit algorithm's performance metric), \cite{han2024ucb} shows the asymptotic normality of sample mean collected from bandit experiments. \cite{fan2021fragility} studies the tail behavior of regret under optimized bandit algorithms, and \cite{simchi2023regret}, \cite{simchi2023multi}, \cite{simchi2023pricing} focus on the interplay between expected regret, regret tail risk, and the statistical power of statistical inference.



Despite these recent advancements, fundamental statistical properties of bandit data still remain largely underexplored. Notably, various numerical {and qualitative} evidence has been reported that in general, adaptively collected data exhibits systematic bias (\cite{xu2013estimation}, \cite{nie2018adaptively}, {\cite{shin2019sample},} \cite{hadad2021confidence}, \cite{dimakopoulou2021online}). Meanwhile, in certain cases, the number of pulls of an arm under popular bandit algorithms heavily fluctuates, as observed by \cite{kalvit2021closer, kuang2024weak}. These phenomena drastically deviate from what one would expect from standard i.i.d. samples. A key feature that sets the bandit data apart from i.i.d. samples is the so-called \textit{sample adaptivity}, which arises since bandit algorithms select the next arm to pull according to the history of all arm's past performance in each step, namely, \textit{fully} adaptive. Such fully-adaptive algorithm induces highly complex dynamics that are challenging to analyze. Consequently, a precise mathematical description of the sample adaptivity under popular bandit algorithms remains lacking in the literature. 

 In this work, we present a novel joint CLT of the number of pulls and the sample mean rewards generated under the celebrated UCB algorithms within the MAB model. This result gives a mathematical characterization of the sample adaptivity of UCB data, shedding light on several important matters including a nonstandard CLT of pseudo-regret and a quantitative characterization of the sample bias (see ``contributions'').

\paragraph{The problem (informal).} 
{Consider a stochastic two-armed bandit instance of length $T$, where each arm $i = 1, 2$ generates rewards according to some arm-specific distribution with mean $\mu_i$. Assume $\mu_1 \geq \mu_2$ without loss of generality. We focus on the class of {\sf generalized UCB1} algorithms that pulls the arm with the highest index $\bar\mu_{i,t-1}+\frac{f(t)}{\sqrt{N_{i,t-1}}}$ at each time $t$, where $N_{i,t-1}$, $\bar\mu_{i,t-1}$ are the number of pulls and sample mean of the collected rewards of arm $i$ at the end of time $t-1$, and $f(\cdot)$ is the exploration function (see Algorithm~\ref{algo: general-ucb1}). 
We are interested in understanding the sample adaptivity of this bandit data, in particular, the correlation structure of the number of pulls and the sample means. To this end, we consider a sequence of such bandit instances by sending $T \to \infty$. The arm gap $\Delta = \mu_1 - \mu_2 $ either remains a constant or $\to 0$ at certain $T$-dependent rate. We aim to characterize (under proper scaling) the joint distribution of $\left(N_{1,T},N_{2,T}, \bar\mu_{1,T}, \bar\mu_{2,T}\right)$ in different asymptotic regimes.}

\subsection{Contributions}\label{sec:contributions} 
We characterize a novel joint CLT of the number of pulls $N_{i,T}$ and the sample mean reward $\bar\mu_{i,T}$ of the arms in a two-armed stochastic bandit environment under the {\sf generalized UCB1} algorithm (Theorem \ref{thm:N1}). For example, (in the simplified setting of unit reward variance for both arms)
\begin{align}\label{eq:main-result-informal}
\begin{pmatrix}
    \Gamma_T\cdot(N_{2, T} - n^{\star}_{2, T}) \\
    \sqrt{n^{\star}_{1,T}}\cdot(\bar\mu_{1,T}-\mu_1) \\
    \sqrt{n^{\star}_{2,T}}\cdot(\bar\mu_{2,T}-\mu_2)
\end{pmatrix}
   \ \xrightarrow[]{d} \ \mathcal{N}\!\left(
  \begin{pmatrix}
    0 \\ 
    0 \\
    0
  \end{pmatrix},
  \begin{pmatrix}
    \lambda^{\star}+1 & {-\sqrt{\lambda^{\star}}} & 1\\
    -\sqrt{\lambda^{\star}} & {1} & 0 \\
    1 & 0 & 1
  \end{pmatrix}
\right)
\end{align}
where $\Gamma_T=\Theta\left(\frac{f(T)}{\nbt}\right)$. Here $n^{\star}_{1,T}$, $n^{\star}_{2,T}$ are the fluid approximation of $N_{1, T}$ and $N_{2, T}$, (see Section \ref{sec: tech_overview} below and Lemma \ref{lem: three gap regimes}), which depend on the arm gap $\Delta$ and $T$. $\lx \in [0, 1]$ captures the proportion of the number of pulls of arm 2 relative to arm 1 in the fluid limit, which is 0 when the arm gap is large and 1 when the arm gap is small. When $\lx$ approaches 0, the correlation between $N_{i,T}$ and the sample mean of the superior arm diminishes. The correlation between $N_{2,T}$ and the sample means in \eqref{eq:main-result-informal} is consistent with what one might expect. Qualitatively, the number of pulls are always positively correlated with the corresponding arm's sample mean and negatively correlated with the other arm's sample mean. 

\eqref{eq:main-result-informal} also generates valuable new insights regarding important matters such as the pseudo-regret and sample bias, novel to the literature. We highlight these findings below.


\paragraph{The non-standard CLT for $N_{i,T}$ and pseudo-regret}
\eqref{eq:main-result-informal} gives us the following non-standard CLT of the number of pulls $N_{i,T}$ for $i = 1, 2$:
\begin{align*}
    \Gamma_T\cdot(N_{i,T}-n^{\star}_{i,T}) \ \xrightarrow[]{d} \ \mathcal N(0,\lx+1)
\end{align*}
where $\Gamma_T=\Theta\left(\frac{f(T)}{n^{\star}_{2,T}}\right)$. 
This CLT interpolates between a standard one when the arm gap is large and a nonstandard one with slow concentration when the arm gap is small. 
In the special case of {\sf UCB1} ($f(T)=\sqrt{2\log T}$), when $\Delta=\Theta(1)$, $N_{2,T}$ concentrates around $\Theta(\log T)$ with typical deviation $\Theta(\sqrt{\log T})$, a common CLT scaling recovering Theorem 6 in \cite{fan2022typical}.
On the other extreme when $\Delta=O(\sqrt{\frac{\log T}{T}})$, $N_{2,T}$ concentrates around $\Theta(T)$ with typical  deviation $\Theta\left(\frac{T}{\sqrt{\log T}}\right)$, a nonstandard CLT with slow concentration as numerically observed by \cite{kalvit2021closer}. 
Our unified CLT characterization for all arm gap regimes bridges the aforementioned two extreme cases with a smooth interpolation. 

Our result on the number of pulls implies the non-standard CLT for pseudo-regret with a typical scaling $\Delta n^{\star}_{2,T}$ and a typical deviation $\Theta\left(\frac{n^{\star}_{2,T}}{f(T)}\Delta\right)$, since the pseudo-regret is simply $\Delta N_{2,T}$. 
Surprisingly, we find that under the algorithm with a faster-growing choice of $f(t)$, both the typical scaling and the typical deviation of pseudo-regret deteriorate. This is in contrast to the tail risk of pseudo-regret, which gets improved by more exploration (see \cite{fan2021fragility}, \cite{simchi2023regret}).


\paragraph{The sample bias}
{\eqref{eq:main-result-informal} implies the asymptotic normality of the sample mean $\bar\mu_{i,T}$ for $i=1,2:$ (with unit reward variance)
\begin{align}\label{eq:sample_mean_CLT_intro}
    \sqrt{n^\star_{i,T}}(\bar\mu_{i,T}-\mu_i)\ \xrightarrow[]{d} \ \mathcal{N}(0,1).
\end{align}
However, this convergence can be slow, a fact not captured by existing theoretical results (\citep{kalvit2021closer, fan2022typical,han2024ucb}). For instance, as Figure \ref{fig:empirical_mean_distribution} shows, in reasonably sized bandit experiments with identical arms (primitives described in the plot), the gap between the standardized empirical distribution of arm 2's sample mean reward and its CLT limit in \eqref{eq:sample_mean_CLT_intro} appears to be much more significant compared with the standard CLT's $\Theta(\frac{1}{\sqrt{n^\star_{i,T}}})$ rate of convergence indicated by the Berry-Esseen theorem.

\begin{figure}[htb!]
	\centering   
\includegraphics[width=0.5\textwidth]{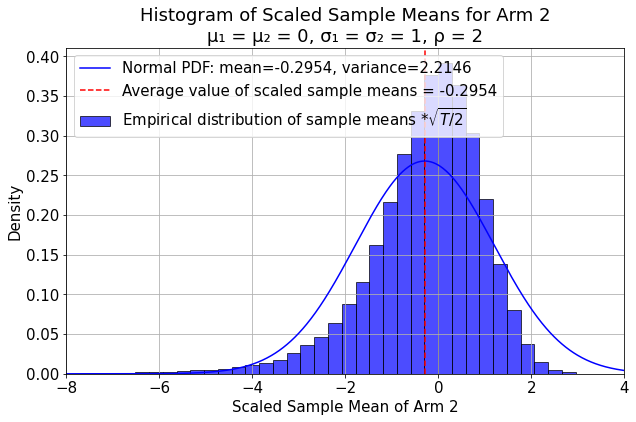}
    \caption{The empirical distribution of the sample mean of arm 2's reward under {\sf UCB1} ($f(t) = \sqrt{\rho\log T}$ with $\rho = 2$) when the horizon length $T=10^5$, with $10^5$ repetitions. Arm $i$'s reward distribution is $\mathcal{N}(\mu_i, 1), i = 1, 2$ (with $\mu_1=\mu_2=0$). The sample mean $\bar\mu_{2,T}$ from each repetition is standardized as in \eqref{eq:sample_mean_CLT_intro}, i.e., scaled by $\sqrt{n^\star_{2,T}}=\sqrt{T/2}$. The normal pdf curve matches the first two moments of the empirical distribution of the scaled sample means.}\label{fig:empirical_mean_distribution}
\end{figure}

Our joint CLT \eqref{eq:main-result-informal} sheds light on the correction of the sample mean's CLT in \eqref{eq:sample_mean_CLT_intro}. In particular, we focus on an important term in statistical inference---the sample bias $\mathbb E[\bar\mu_{i,T}]-\mu_i$. It is generally known that adaptively collected data may exhibit sample bias due to the correlation between the sample size and the sample mean (cf. \cite{bowden2017unbiased}), yet there lacks a theoretical characterization of the sample bias of data collected under popular bandit algorithms. Our joint CLT reveals the correlation structure between the number of pulls and the sample mean reward of an arm, enabling us to heuristically quantify the sample bias under UCB algorithms 
at an asymptotic precision beyond the CLT scaling in \eqref{eq:sample_mean_CLT_intro}.
For example, consider bandit data under the canonical {\sf UCB1} algorithm ($f(t)=\sqrt{\rho\log t}$) when the arm gap is zero ($\Delta = 0$). In this case, the two arms are identical, and the fluid number of pulls are $\nat = \nbt = \frac{T}{2}.$ We conjecture that (assuming unit reward variance) for both $ i = 1, 2$
\begin{align}\label{eq:sample-bias-intro}
    \sqrt{\frac{T}{2}}\left(\mathbb E[\bar\mu_{i,T}]-\mu_i\right)=
        -\sqrt{\frac{1}{\rho \log T}} + o\left(\frac{1}{\sqrt{\log T}}\right). 
\end{align}
{Numerical results in Appendix \ref{app:sampling_bias} compare the conjectured sample bias with the empirical sample bias from repeated experiments, which indicate the effectiveness of our conjecture.} We highlight that the {conjectured} sample bias is negative as qualitatively reported by \cite{nie2018adaptively} for adaptive data collection, 
and it vanishes at a slow rate of $\frac{1}{\sqrt{\log T}}$ (after normalization). This slow decay indicates that the bias remains significant enough that standard confidence intervals and inference methods based on the CLT may not be valid, especially when the sample size is not large enough. In Section \ref{sec:bias} we provide a complete characterization of our conjectured sample bias in all arm gap regimes beyond $\Delta = 0$.
}

\subsection{Technical overview}\label{sec: tech_overview}
We introduce a novel analysis framework to establish our main results, that views the bandit process as a complex dynamical system, and conducts perturbation analysis on top of it. The approach is generic and broadly applicable, of which we provide an overview within the $K$-armed bandit system and for UCB-type algorithms with general index functions $I(\cdot)$.

\paragraph{A Perturbation Analysis} Generally, any index policy such as {\sf UCB1} adaptively selects the next arm $i$ to pull based on the highest index $I(\bar\mu_{i,t},N_{i,t},t)$ for some index function $I(\cdot)$. In a continuous-time fluid approximation, we replace the stochastic reward by its mean, and let the index policy to continuously pull the arm with the highest index. If $I$ is smooth and satisfies the natural exploration-encouraging conditions (i.e., for any arm in the fluid system, the index increases in time $t$ whenever it is \textit{not} pulled, and decreases otherwise), then all arms' indices will always be kept equal under the algorithm. This gives the following natural characterization of the fluid system (at time $T$):
\begin{align}\label{eq: fluid system}
     I\left(\mu_1, n_{1, T}, T\right) &= I\left(\mu_k, n_{k, T}, T\right),\ \ 2\leq k \leq K\nonumber\\
     \sum_{k = 1}^K n_{k, T} &= T.
\end{align} 
Intuitively, the solution to the above system of equations, denoted by $\bn^{\star}_T \triangleq \left(n^{\star}_{1, T}, \dots n^{\star}_{K, T}\right)$ is expected to be a first-order approximation of $(N_{1,T},...,N_{K,T})$, the true number of pulls, under proper conditions, as observed and formalized in earlier works \cite{kalvit2021closer, han2024ucb}. We refer to $n^{\star}_{1, T}, \dots n^{\star}_{K, T}$ as the \textit{fluid approximation} of the number of pulls of each arm.

This work goes beyond \eqref{eq: fluid system}, and reveals how the true system’s dynamics deviate from the fluid approximation through a perturbation analysis. Inspired by the system of equations \eqref{eq: fluid system}, a natural conjecture is that the {\sf UCB} algorithm in the true system also tries to ``equate the indices''. 
\begin{conj}[Informal]\label{conj: true system}
In a ``reasonable'' bandit model and under a ``reasonable'' {\sf UCB} algorithm with index function $I$, the number of pulls $N_{k, T}$ should satisfy
\begin{align}\label{eq: true system conjectured fixed point equation}
    I\left(\bar\mu_{1, T}, N_{1, T}, T\right) &\approx  I\left(\bar\mu_{k, T}, N_{k, T}, T\right),\ \ 2\leq k \leq K,\nonumber\\
    \sum_{k = 1}^T N_{k, T} &= T.
\end{align} 
\end{conj}
Replace $I(\cdot)$ by their first-order approximations at $(\mu_k, n^{\star}_{k, T})$, namely
\begin{align*}
    I\left(\bar\mu_{k, T}, N_{k, T}, T\right) \approx I\left(\mu_{k}, n^{\star}_{k, T}, T\right) + I'_{k, 1} \cdot (\underbrace{\bar\mu_{k, T} - \mu_k}_{\bar\varepsilon_k}) + I'_{k, 2} \cdot (\underbrace{N_{k, T} - n^{\star}_{k, T}}_{\omega_k}),
\end{align*}
where $I'_{k, 1}, I'_{k,2}$ are the partial derivatives of $I$ w.r.t. the first and second arguments evaluated at $(\mu_k,n^{\star}_{k,T},T)$, respectively. This approximation, combined with \eqref{eq: fluid system}, allows us to further simplify the conjectured system of equations \eqref{eq: true system conjectured fixed point equation}, to a system of linear equations
\begin{align}\label{eq: linear-system-correlation-conjecture}
    \begin{bmatrix}
1 & 1 & 1 & \dots & 1 \\
-I'_{1,2} & I'_{2,2} & 0 & \dots & 0 \\
\dots  & \dots  & \dots  & \dots & \dots  \\
-I'_{1,2} & 0 & 0 & \dots & I'_{K,2}
\end{bmatrix}
\begin{bmatrix}
\omega_1 \\ \omega_2 \\ \dots \\ \omega_K
\end{bmatrix}
=
\begin{bmatrix}
0 \\ I'_{1,1}\bar\varepsilon_1  - I'_{2,1}\bar\varepsilon_2 \\ \dots \\ I'_{1, 1}\bar\varepsilon_1 - I'_{K,1}\bar\varepsilon_K
\end{bmatrix}.
\end{align}
which admits closed-form solutions (Refer to Lemma~\ref{lem: linear-system-correlation-conjecture} in Appendix \ref{app:k-arm} for a precise form). In particular, we obtain approximations of $N_{k, T}, k = 1, \dots K$ beyond their fluid approximations $\bn^{\star}_T$, each as an linear combination of the sample means $\bar\mu_{k, T}, k = 1, \dots K$. This characterizes the dependence structure between $N_{k, T}$ and $\bar\mu_{k, T}$ for $k = 1, \dots K$. We then arrive at the joint CLT of \eqref{eq:main-result-informal} (and Theorem \ref{thm:N1}) through an approximation of each $\bar\mu_{k, T}$ by the sample means of the corresponding arm at their fluid approximation $n^{\star}_{k, T}$, which are non-adaptive, completely independent and asymptotically normal across arms $k = 1, \dots, K$.

The main result of this work, Theorem \ref{thm:N1}, can be viewed as a formalization of the above intuition in the two-arm case, with the ``reasonable'' bandit model, the ``reasonable'' {\sf UCB} algorithm, as well as the precise notion of ``approximately equal'' rigorously specified. We expect that similar results hold in the general $K$-arm setting. A precise form of the joint CLT, derived following the procedures sketched above, is provided in Appendix~\ref{app:k-arm}, with brief discussion of its implications, although a formal proof is omitted.

\paragraph{Additional notation.}
For a sequence of random variables $Y_n$, we denote by $Y_n \xlongrightarrow[]{p} Y$ and $Y_n \xlongrightarrow[]{d} Y$, respectively, the convergence in probablity and in distribution. We say $f(T)=o(g(T))$ or $g(T)=\omega(f(T))$ if $\lim_{T\rightarrow\infty}\frac{f(T)}{g(T)}=0$. Similarly, $f(T)=O(g(T))$ or $g(T)=\Omega(f(T))$ if $\limsup_{T\rightarrow\infty}\Big|\frac{f(T)}{g(T)}\Big|\leq C$ for some constant $C$. If $f(T)=O(g(T))$ and $f(T)=\Omega(g(T))$ hold simultaneously, we say $f(T)=\Theta(g(T))$. We write $f(T)\sim g(T)$ in the special case where $\lim_{T\rightarrow\infty}\frac{f(T)}{g(T)}=1$. If either sequence $f(T)$ or $g(T)$ is random, and one of the aforementioned ratio conditions holds in probability, we use the subscript $p$ with the corresponding Landau symbol. For example, $f(T)=o_p(g(T))$ if $\frac{f(T)}{g(T)}\xrightarrow[]{p}0$ as $T\rightarrow\infty$. 
Similar to $\bn$, we in general use bold symbols to denote vectors, e.g. $\bN_{j} = (N_{1, j}, \dots, N_{K, j})$ and $\bar\bmu_{j}= (\bar\mu_{1, j}, \dots, \bar\mu_{K, j})$. 

\paragraph{Organization of the paper} We formally setup the problem in Section~\ref{sec:prelim}. The main result is presented in Section \ref{sec:main-result} and its implications are discussed in Section \ref{sec:implications}.

\section{Preliminary}\label{sec:prelim}
\paragraph{The MAB model.\ } 
 We consider a sequence of two-armed stochastic bandit problems, indexed by $T \geq 1$. The $T^{\rm th}$ problem has $T$ decision epochs. Associated with each arm $i\in\{1,2\}$ in the $T^{\rm th}$ problem is a reward distribution $\cP^T_i$ with mean $\mu^T_i$, and an infinite sequence of rewards $X^T_{i, 1},  \dots$ drawn \emph{i.i.d.} from $\mathcal{P}^T_i$. Let $\bar\mu^T_{i}(m) \triangleq \frac{1}{m}\sum_{j = 1}^m X^T_{i, j}$ be the (running) sample mean of arm $i$'s reward with the first $m$ samples. Denote by $\cP^T_\star$ the reward distribution with the largest mean $\mu^T_\star$. Let $\Delta^T_i \triangleq \mu^T_\star - \mu^T_i$ denote the sub-optimality gap of arm $i$. WLOG we let $\cP^T_\star = \cP^T_1$, i.e. arm $1$ is the best arm (with the largest mean) for all $T \geq 1$. 

We impose the following assumptions on the bandit environment.
\begin{assumption}[Properties of the bandit environment]\label{assump: distribution}
The reward distributions satisfy:
\begin{enumerate}
    \item $\mu^T_i$ is uniformly bounded for each $i = 1, 2$, and $\Delta^T$ is monotone decreasing in $T$. 
    \item $\textrm{Var}(Y) = (\sigma^T_i)^2$ exists for $Y$ distributed according to $\cP^T_i$. Furthermore there exists positive constants $\sigma_1, \sigma_2$ and $\sigma$, such that $\lim_{T \to \infty} \sigma^T_i = \sigma_i$ and $\sigma_i \leq \sigma, i = 1, 2.$
    \item  $\cP^T_i$ are sub-Gaussian for each $i = 1, 2$ and any $T \geq 1$.
\end{enumerate}
\end{assumption}
\begin{remark}
We allow $\Delta^T \not\to 0$, which effectively captures the ``constant-gap'' regime.
\end{remark}

\paragraph{The {\sf generalized UCB1} algorithms}

In this work, we focus on the {\sf generalized UCB1} algorithm, which generalizes the celebrated and widely studied {\sf UCB1} algorithm (\cite{Auer2002}). Formally, in the $T^{\rm th}$ bandit problem, {\sf generalized UCB1} with exploration function $f(t)$ selects an arm $A_t = i \in \{1, 2\}$ with the highest index $\bar\mu^T_{i}(N^T_{i, {t-1}}) + (N^T_{i, {t-1}})^{-\frac{1}{2}}f(t)$ at a decision epoch $t$, upon which the next not-yet-revealed reward in the sequence $X^T_{A_t, 1}, X^T_{A_t, 2} \dots$ is revealed and collected by the algorithm. Here $N^T_{i, t} \triangleq \sum_{j = 1}^t \mathbbm{1}_{\left\{A_j = i\right\}}$ denote the number of pulls of arm $i$ up to {(and including)} time $t$. To simplify notation, we use $\bar\mu^T_{i, {t-1}} \triangleq \bar\mu^T_{i}(N^T_{i, {t-1}})$ to denote the sample mean of arm $i$'s rewards at {the beginning of} decision epoch $t$. Furthermore, we drop the superscript $T$ and use notations $X_{i,j}, N_{i,j}, \bar\mu_i(m), \bar\mu_{i,j}$ instead when $T$ is clear from the context. A formal description of the {\sf generalized UCB1} algorithm is given in Algorithm \ref{algo: general-ucb1}.


\begin{algorithm}
    \begin{algorithmic}[1]
        \caption{The {\sf generalized UCB1}}
        \State \textbf{Input:} Exploration function  $f(\cdot)$.
        \State At $t = 1, 2$, play each arm $i$ once and initiate $N_{i,2} = 1, \bar\mu_{i,2}=X_{i,1}$, $i \in \{1, 2\}$.
        \State \textbf{for} $t \in \{3, ..., T\}$ \textbf{do}
        \State \quad\quad Select arm $A_t \in \arg\max_{i \in \{1, 2\}} \left\{\bar\mu_{i, t-1}+\frac{f(t)}{\sqrt{N_{i, t-1}}}\right\}$.
        \State \quad\quad Update $N_{i, t} \leftarrow N_{i, t-1} + \mathbbm{1}_{\left\{A_t = i\right\}}$.
        \State \quad\quad Update $\bar\mu_{i, t} \leftarrow  \frac{ \bar\mu_{i, t-1} N_{i, t-1}  + X_{i, N_{i, t}} \mathbbm{1}_{\left\{A_t = i\right\}} }{N_{i, t}} $.
        \label{algo: general-ucb1}
    \end{algorithmic}
\end{algorithm}

We specify some technical assumptions on the exploration function $f(\cdot)$.
\begin{assumption}[Properties of the exploration function]\label{assump: index}
The exploration function $f(t)$ satisfies the following conditions
        \begin{enumerate}
        \item $f(t)$ is monotone increasing and $f(t) = \omega(\sqrt{\log\log t})$
            \item There exists $0 \leq \beta < \frac{1}{2}$, such that $\frac{f(t)}{t^\beta}$ is decreasing in $t$.
        \end{enumerate}
\end{assumption}
\begin{remark}
 $f(t) = \sqrt{\rho\log T}$ for some constant $\rho$ recovers the {\sf canonical UCB} of \cite{kalvit2021closer}. In particular, when $\rho = 2$, we recover the {\sf UCB1} of \cite{Auer2002}. In general, $f(t)$ is allowed to scale in a broad range, faster than $\sqrt{\log\log t}$ and slower than $\sqrt{t}$.
\end{remark}

\section{Main Result}\label{sec:main-result}
 Under the {generalized UCB1} class of algorithms, the generic fluid systems of equations Eq. (\ref{eq: fluid system}) has the following explicit form
\begin{align}\label{eq: fluid equation 2 arm f}
    (\nbt)^{-\frac{1}{2}} - (\nat)^{-\frac{1}{2}} = (f(T))^{-1}\Delta^T \ \ , \ \  \nat + \nbt = T,
\end{align}
where we denote $\Delta^T \triangleq \Delta^T_2$ to be the mean gap between the two arms to simplify notation. The form of the fluid equations leads to the following explicit scaling characterization of $\nat, \nbt$, in three different regimes.  
\begin{lemma}[Fluid Scaling]\label{lem: three gap regimes}
Let $(\nat, \nbt)$ be the unique solution of Eq. \eqref{eq: fluid equation 2 arm f}. Denote by $\lx \triangleq \lim_{T \to \infty}\frac{\nbt}{\nat}$. The scaling of $(\nat, \nbt)$ and $\lx$ can be explicitly specified. Precisely,
\begin{itemize}
    \item ``Large gap'': $\Delta^T = \omega\left(\frac{f(T)}{\sqrt{T}}\right)$, then $\nbt \sim \left(\frac{f(T)}{\Delta^T}\right)^2\ , \ \nat \sim T , \ \lx = 0.$
    \item ``Small gap'': $\Delta^T = o\left(\frac{f(T)}{\sqrt{T}}\right)$, then $\nbt \sim \frac{T}{2}\ , \ \nat \sim \frac{T}{2} , \ \lx = 1.$
    \item ``Moderate gap'': $\Delta^T \sim \theta\frac{ f(T)}{\sqrt{T}}$ for some $\theta \geq 0$, then $\nbt \sim \frac{\lx}{1 + \lx} T \ , \ \nat \sim \frac{1}{1 + \lx} T$ with $\lx \in (0, 1]$ solves $\sqrt{1 + \frac{1}{\lx}} - \sqrt{1+ \lx} = \theta.$
\end{itemize}
\end{lemma}
We omit the proof.  Lemma~\ref{lem: three gap regimes} gives the first-order characterization of the dynamics of the UCB algorithm.  As our main result, we describe how the true bandit system under UCB algorithms fluctuates around the fluid approximation. The quantities $\nat, \nbt, \lx$ are thus crucial in our main result, which is stated below.

\begin{theorem}[Joint CLT]\label{thm:N1}
Consider a two-armed bandit environment satisfying Assumption \ref{assump: distribution}. The {\sf generalized UCB1} in Algorithm \ref{algo: general-ucb1} with exploration function $f(t)$ that satisfies Assumption \ref{assump: index} is implemented.
Then 
\begin{align*}
\begin{pmatrix}
    \frac{1 + \left(\lx\right)^{\frac{3}{2}}}{2} \frac{f(T)}{\nbt}\left(N_{2, T} - \nbt\right) \\ 
    \sqrt{\nat}\left(\bamu_{1, T} - \mu^T_1\right) \\
    \sqrt{\nbt}\left(\bamu_{2, T} - \mu^T_2\right)
  \end{pmatrix} \ \xlongrightarrow[]{d} \ \mathcal{N}\!\left(
  \begin{pmatrix}
    0 \\ 
    0 \\
    0
  \end{pmatrix},
  \begin{pmatrix}
   \lx\sigma_1^2 + \sigma_2^2 & -\sigma_1^2\sqrt{\lx} & \sigma_2^2 \\
    -\sigma_1^2\sqrt{\lx} & {\sigma_1^2} & 0 \\
    \sigma_2^2 & 0 & \sigma_2^2
  \end{pmatrix}
\right),
\end{align*}
where $\nat, \nbt, \lx$ is defined as in Lemma~\ref{lem: three gap regimes}.  
\end{theorem}
\begin{remark}
    Theorem~\ref{thm:N1} can be equivalently stated as a four-dimensional joint CLT with the addition of the number of superior arm pulls, $N_{1, T}$, as (trivially) $N_{1, T} - \nat = -\left(N_{2, T} - \nbt\right)$. We state it in the current form for ease of notation.
\end{remark}

The observations made in prior works \cite{kalvit2021closer, han2024ucb} regarding the statistical amenability of {\sf UCB1} data can be recovered with Theorem~\ref{thm:N1}. Firstly, since $f(T) = \omega(1)$ and $\nat \geq \nbt$, the weak LLN $\frac{N_{i, T} - n^{\star}_{i, T}}{n^{\star}_{i, T}} \xlongrightarrow[]{p} 0$ follows directly from Theorem~\ref{thm:N1}. In words, the number of arm pulls are asymptotically concentrated around the respective fluid approximation regardless of the mean gap regime. Secondly, the naive mean estimator $\bamu_{i, T} = \bamu_i(N_{i, T})$ is asymptotically unbiased and enjoys the CLT with standard deviation $\Theta\left((n^{\star}_{i, T})^{-\frac{1}{2}}\right)$—in words, as if they were computed from standard \emph{i.i.d.} samples. 

What is more interesting, however, is the additional message delivered by Theorem \ref{thm:N1}. First, we establish a non-standard CLT for the number of pulls, with standard deviation $\Theta\left(\frac{\nbt}{f(T)}\right)$ instead of the common $\Theta\left(\sqrt{\nbt}\right)$ scaling one would expect. Second, we explicitly characterize the asymptotic correlation between the number of pulls and the sample means across different asymptotic regimes. Qualitatively, the number of pulls are always positively correlated with the corresponding arm's sample mean and negatively correlated with the other arm's sample mean, consistent with what one might expect. Moving from the moderate and small gap regimes to the large gap regimes, the impact of the superior arm (arm 1)'s performance fluctuation on the number of pulls diminishes. Both of these findings are novel to the literature.


The data generated by online learning algorithms/sequential experiments is generally known to be deviating from the standard \emph{i.i.d.} samples due to {sample adaptivity}. Theorem~\ref{thm:N1} provides the first mathematical characterization of such {sample adaptivity} for the celebrated  UCB algorithms. In the next section, we leverage Theorem~\ref{thm:N1} to show that the amenable properties of bandit data collected by UCB algorithms mentioned above, namely, the WLLN of the number of pulls and the CLT of the naive mean estimators, in fact, both {suffer from slow convergence and can be problematic on reasonable-sized data}. We also discuss the implication of Theorem~\ref{thm:N1} on the algorithm's pseudo-regret. The proof of Theorem~\ref{thm:N1} is deferred to Appendix~\ref{app: proof of main thm}.

\begin{remark}[Extension to $K$ arms]
An extension of Theorem~\ref{thm:N1} to the $K$-arm setting is provided in Appendix \ref{app:k-arm}. In general, the precise correlation structure among the number of pulls and the sample means depend on the mean gap scaling of \textit{all} arms, with a complicated form. In certain special regimes (of arm's mean gap), the general form of the CLT can be simplified.  See Appendix~\ref{app:k-arm} for more discussion. 
\end{remark}

\section{Implications}\label{sec:implications}
\subsection{The non-standard CLT for $N_i(T)$ and pseudo-regret}\label{sec: regret}
Focusing on the marginal distribution of the number of pulls, Theorem \ref{thm:N1} yields  for $i = 1, 2$
\begin{align}\label{eq: N-marginal ClT}
    \frac{1 + \left(\lx\right)^{\frac{3}{2}}}{2} \frac{f(T)}{\nbt}\left(N_{i, T} - n^{\star}_{i, T}\right) \ \xlongrightarrow[]{d} \mathcal{N}(0, \lx \sigma_1^2 + \sigma_2^2),
\end{align}     
valid in all arm gap regimes, for all {\sf UCB} exploration functions $f(t)$ that satisfy Assumption \ref{assump: index} and for all bandit environments that satisfy Assumption \ref{assump: distribution}. The only existing result of this type was provided in \cite{fan2022typical}, in the \textit{constant-gap} setting for Gaussian rewards, under the {\sf UCB1} algorithm ($f(t) = \sqrt{2 \log t}$). Notice that the constant-gap setting is a special (in fact, extreme) case in the large-gap regime, with $\lx = 0$ and $\nbt \sim \frac{2 \log T}{\Delta^2}$, hence (\ref{eq: N-marginal ClT}) becomes (for arm 2)
\begin{align*}
    \frac{\Delta^2}{2\sqrt{2\log T}}\left(N_{2, T} - \frac{2 \log T}{\Delta^2}\right)  \ \xlongrightarrow[]{d} \mathcal{N}(0, \sigma_2^2),
\end{align*}
effectively recovering Theorem 6 in \cite{fan2022typical}. In the other extreme, namely the moderate-to-small gap regime, both arms will get a non-trivial proportion ($\Theta(T)$) number of pulls. \cite{kalvit2021closer} studies this regime, where they proved the weak LLN of $\frac{N_{i, T}}{T}$ under {\sf canonical UCB} ($f(t) = \sqrt{\rho \log t}$ for some constant $\rho$) for bounded rewards bandit, with a $\Theta\left(\sqrt{\frac{\log\log T}{\log T}}\right)$ conjectured convergence rate yet without proof. The subsequent work \cite{han2024ucb} nudges one side of this conjecture with an $o\left(\sqrt{\frac{\log\log T}{\log T}}\right)$ guarantee for the convergence rate for bandits with Gaussian rewards under a $T$-aware simplified version of {\sf UCB1} (cf. Theorem 3.6 in \cite{han2024ucb}). By contrast, (\ref{eq: N-marginal ClT}) provides the first CLT-type characterization of $N_{2, T}$ in such regimes, 
implying the correct, accurate convergence rate of $\Theta\left(\frac{1}{\sqrt{\log T}}\right)$. This places the conjecture of \cite{kalvit2021closer} on the marginally pessimistic side. 

The above demonstrates sharply contrasting behavior of the {\sf UCB1} algorithm in terms of the number of inferior arm pulls in different regimes. In the constant-gap setting, $N_{2, T}$ is asymptotically concentrated around $\frac{2}{\Delta^2}\log T$ with standard deviation $\Theta(\sqrt{\log T})$, a ``standard'' CLT scaling. However, in the moderate-small gap regimes, $N_{2, T}$ is asymptotically concentrated around $\frac{\lx}{1 + \lx} T$ with standard deviation $\Theta(\frac{T}{\sqrt{\log T}})$. This is a non-standard CLT scaling, where the extremely slow rate of $\Theta(\frac{1}{\sqrt{\log T}})$ necessitates a very large $T$ in order for the LLN concentration of $N_{i, T}$ to become apparent. Such a \textit{slow-concentration} phenomenon was observed numerically and reported in \cite{kalvit2021closer}.


Our unified CLT of (\ref{eq: N-marginal ClT}) effectively bridges the performance of {\sf UCB1} in the aforementioned two extreme cases through a smooth interpolation across varying mean gap regimes, under much more generalized settings (both in terms of algorithm and bandit environment). Moreover, the different CLTs under different algorithms in the considered class provide additional insights for algorithmic design through the implied distribution of the pseudo-regret.

\paragraph{Pseudo-regret: typical scale and deviation}
The \textit{pseudo-regret} of an algorithm is defined as $\bar R_T\triangleq\mu_1 T-\sum_{t=1}^T\mu_{A_t}$ (see, e.g., \cite{lattimore2020bandit}). While the vast majority of the bandit literature focus on bounding the expected regret $\E[\bar R_T]$, there is a recent surge of interests in understanding $\bar R_T$, in particular, its distributional properties. (see introduction) Observe that $\bar R_T=N_{2,T}\Delta_T$. Thus, our characterization of the asymptotic normality of $N_{2, T}$ directly implies a CLT for the pseudo-regret.
\begin{corollary}\label{cor:regret CLT}
The pseudo-regret of {\sf generalized UCB1} satisfies
    \begin{align*}
    \frac{1 + \left(\lx\right)^{\frac{3}{2}}}{2} \frac{f(T)}{\nbt\Delta^T}\left(\bar R_T - n^{\star}_{2, T}\Delta^T\right) \ \xlongrightarrow[]{d} \mathcal{N}(0, \lx \sigma_1^2 + \sigma_2^2),
\end{align*}
\end{corollary}
Corollary \ref{cor:regret CLT} implies $\bar R_T\sim_p \nbt \Delta^T$. 
Under {\sf UCB1} ($f(t) = \sqrt{2 \log T}$), the scaling of $n_{2,T}^{\star}\Delta^T$ aligns with the celebrated \textit{instance-dependent} and \textit{minimax} regret scaling: in the constant-gap regime, $\Delta^T = \Delta > 0$ and $n_{2,T}^{\star}\Delta^T  =  \Theta\left(\log T\right)$; while in the moderate-gap regime, $\Delta^T =  \Theta\left(\sqrt{\frac{\log T}{T}}\right)$, and $ n_{2,T}^{\star}\Delta^T = \Theta\left(\sqrt{T \log T}\right)$. We shall refer to $\nbt \Delta^T \triangleq R^{\star}_T$ as the \textit{typical scale} of $\bar R_T$. 

Beyond the typical scale, Corollary~\ref{cor:regret CLT} also characterize the asymptotic standard deviation of $\bar R_T$, which is of the form $\frac{\sqrt{2\left(\lx \sigma_1^2 + \sigma_2^2\right)}}{\sqrt{1 + \left(\lx\right)^{\frac{3}{2}}}}\frac{{R^{\star}_T}}{f(T)} \triangleq S^{\star}_T$, referred to as the \textit{typical deviation}. Under {\sf UCB1}, in the constant-gap regime $S^{\star}_T = \Theta\left(\sqrt{\log T}\right)$, and in the moderate-gap regime $S^{\star}_T = \Theta\left(\sqrt{T}\right)$. We observe an undesirably high typical deviation in the moderate regime (in line with the slow concentration of $N_{2, T}$). 

In general, Corollary~\ref{cor:regret CLT} implies that $S^{\star}_T = \Theta\left(R^{\star}_T/f(T)\right)$. This allows us to investigate how \textit{algorithmic design} (within the {\sf generalized UCB1} class) impacts the resulting pseudo-regret, in terms of both typical scale and the typical deviation. At first glance, one might think that a faster-growing exploration function $f(t)$ helps reduce the typical deviation yet hurts the typical scale, leading to a trade-off between the two objectives. This is, quite surprisingly, not the case.
\begin{prop}\label{prop: risk-regret-change-f}
    Suppose $f$ and $g$ satisfy Assumption \ref{assump: index} with $g(T) = \Omega(f(T))$. Consider the pseudo-regret under the corresponding UCB algorithms, and denote $R^{\star, f}_T, R^{\star, g}_T$ their typical scale, and $S^{\star,f}_T, S^{\star,g}_T$ their typical deviation, respectively. Then $R^{\star, g}_T = \Omega\left(R^{\star, f}_T\right)$, and $S^{\star,g}_T = \Omega\left(S^{\star,f}_T\right)$ for any arm gap regime.
\end{prop}
Proposition~\ref{prop: risk-regret-change-f} follows from {Corollary~\ref{cor:regret CLT} and} Lemma~\ref{lem: three gap regimes}. It implies that a faster-growing $f(t)$ results in algorithmic performance deterioration in terms of both the typical scale and typical deviation of the pseudo-regret. In principle, this strongly motivates the choice of exploration function $f(t)$ to be as slow-growing as possible, where we note that Assumption \ref{assump: index} allows for a minimal rate of $\omega(\log\log t)$. 

However, a choice of $f$ that grows too slowly comes with the cost of potentially hurting other algorithmic objectives, namely, the expected regret. Indeed, for {\sf generalized UCB1} with exploration function $f(t) = o(\sqrt{\log t})$ (yet satisfying Assumption \ref{assump: index}), Corollary~\ref{cor:regret CLT} continues to guarantee that $\bar R_T \sim_p R^{\star}_T = \Theta((f(T))^2) = o(\log T)$ in the constant-gap regime. However, the celebrated Lai and Robbins' lower bound implies that the expected regret $\E[\bar R_T]$ cannot achieve universal $o(\log T)$ scaling. The discrepancy suggests a separation between the typical scale and the expected value of $\bar R_T$, which is due to the \textit{atypical} deviation of $\bar R_T$ from its typical scale with a relatively large (while still vanishing) probability.

The current work focuses only on the ``typical scenarios'', capturing the $(1 - \epsilon)$-high probability behavior of {\sf generalized UCB1} as $T$ scales for any fixed $\epsilon > 0$. This separates us from the line of work studying the ``atypical scenarios'' that occurs with vanishing probability, e.g., those on the large-deviation tail risks of algorithms (cf. \cite{fan2021fragility}, \cite{simchi2023regret}). 





\subsection{The sample bias}\label{sec:bias} 
Beyond the marginal distributions, Theorem~\ref{thm:N1} also provides an explicit correlation structure between the number of pulls and the sample means, which characterizes the \textit{sample-adaptivity} in data generated by UCB algorithms, and, more importantly, offers insights into the corresponding statistical inference tasks performed on such samples. Inspired by Theorem~\ref{thm:N1}, we construct a stylized data-generating model,  which (i) is easy to describe and analyze (with only one level of adaptivity), and (ii) well approximates the sample adaptivity of the true (fully adaptive) data generated from the {\sf generalized UCB}. In particular,  this stylized model suggests a particular scale of the bias of the naive mean estimator, which we verify numerically to well predict the true bias on UCB data.\\


\noindent\textbf{A stylized data-generating model}\\
Initiate: A sequence $\delta_T: \delta_T = \omega\left((f(T))^{-1}\right)$ and $\delta_T = o(1)$. 
\begin{enumerate}
    \item Generate $n^{\delta}_{i, T} \triangleq (1 - \delta_T)n^{\star}_{i, T}$\ \ \emph{i.i.d.} rewards from arm $i$, $i = 1, 2$
    \item Compute the normalized sample mean from the two arms:
    \begin{align*}
        Z^{\delta}_{i, T} \triangleq \sqrt{n^{\delta}_{i, T}}\left(\bamu^T_i\left(n^{\delta}_{i, T}\right) - \mu^T_{i}  \right), \ i = 1, 2.
    \end{align*}
    \item Compute 
    \begin{align}\label{eq:N_2_first_order_approximation}
        \tilde N_{2,T}&= \nbt\left(1 +\frac{2\left(Z^{\delta}_{2, T} - Z^{\delta}_{1, T}\sqrt{\lx}\right)}{\left(1 + (\lx)^{\frac{3}{2}}\right)f(T)}\right), \ \ \tilde N_{1, T} = T - \tilde N_{2, T}.
    \end{align}
    \item Sample $\tilde N_{i, T} - n^{\delta}_{i, T}$ more \emph{i.i.d.} rewards from the two arms, respectively.
\end{enumerate}
We denote the sample mean in this stylized model $\tilde \mu_{i, T}$, respectively for the two arms. 
{We argue that $\tilde\mu_{i,T}$, as a random variable, is a good approximation of $\bar\mu_{i,T}$ to reflect the latter's first-order bias, since the construction of $\tilde\mu_{i,T}$ captures the first-order correlation between sample mean and sample size. To see this, note that by Theorem \ref{thm:N1}, $N_{i,T}$ is asymptotically concentrated around $n^{\star}_{i,T}$ with a typical deviation of $\Theta(\frac{n^{\star}_{2,T}}{f(T)})$, hence w.h.p., $N_{i,T}>n^\delta_{i,T}$ (where $n^\delta_{i,T}$ is defined in Step 1 of the above stylized model). Therefore, the sample size of data collected for arm $i$ is w.h.p. at least the deterministic quantity $n^\delta_{i,T}$, and these data are i.i.d. with an unbiased sample mean $\bar\mu_i^T(n^\delta_{i,T})$ (see Step 2 of the stylized model). The sampling bias in the real sample mean $\bar\mu_{i,T}$ comes from the correlation between $\bar\mu_i^T(n^\delta_{i,T})$ and the number of additional samples. Theorem \ref{thm:N1} further implies that the number of additional samples can be approximated from the values of $\bar\mu_i^T(n^\delta_{i,T})$, $i=1,2$. Step 3--4 of stylized model calculates the number of additional samples. Note that \eqref{eq:N_2_first_order_approximation} in Step 3 is simply derived from Theorem \ref{thm:N1}, with $\bar\mu_{i,T}$ replaced by $\bar\mu_i^T(n^\delta_{i,T})$. This replacement is legitimate by Lemma \ref{lem: high-prob-N1-n1} in the appendix.}

We defer a more detailed derivation of the sampling bias in the above stylized model to Appendix \ref{app:sampling_bias}. The explicit bias term well approximates the sample bias under a {\sf canonical UCB} algorithm (with $f(t) = \sqrt{\rho \log t}$), which is the content of the next conjecture.

\begin{conj}
\label{conjecture:sampling-bias-UCB}
    Suppose data are generated by a canonical UCB1 algorithm with exploration function $f(t)=\sqrt{\rho\log t}$ in a two-arm stochastic bandit environment. Consider the sample mean $\bar\mu_{i,T}$ of arm $i$, $i=1,2$. Then 
        \begin{itemize}
        \item ``Large gap:'' If $\Delta^T=\omega\left(\sqrt{\frac{\log T}{T}}\right)$, then
        \begin{align*}
            &\mathbb E[\bar \mu_{1,T}]=\mu_1^T+O\left(\frac{\log T}{T}\right),\\
            &\mathbb E[\bar \mu_{2,T}]=\mu_2^T-\frac{2\sigma_2^2\Delta^T}{\rho\log T}+o\left(\frac{\Delta^T}{\log T}\right).
        \end{align*}
        \item ``Moderate/small gap'': If 
        $\Delta^T = O\left(\sqrt{\frac{\log T}{T}}\right)$ then
        \begin{align*}
            &\mathbb E[\bar\mu_{1,T}]=\mu_1^T- \frac{2\sigma_1^2\sqrt{1 + \lx}}{\sqrt{\rho}\left(1 + (\lx)^{-\frac{3}{2}}\right)}\frac{1}{\sqrt{T \log T}}+o\left(\frac{1}{\sqrt{T\log T}}\right),\\
            &\mathbb E[\bar\mu_{2,T}]=\mu_2^T-\frac{2\sigma_2^2\sqrt{1 + \lx}}{\sqrt{\rho}\left(\sqrt{\lx} + (\lx)^2\right)}\frac{1}{\sqrt{T\log T}}
             +o\left(\frac{1}{\sqrt{T\log T}}\right).
        \end{align*}
    \end{itemize}
\end{conj}
One can compare the sample bias in Conjecture \ref{conjecture:sampling-bias-UCB} with the sample mean's CLT in Theorem \ref{thm:N1}, restated below:
\begin{align}\label{eq:sample-mean-CLT}
    \sqrt{n^\star_{i,T}}(\bar\mu_{i,T}-\mu_i^T)\xlongrightarrow[]{d}\mathcal N(0,\sigma_i^2).
\end{align}
In contrast, Conjecture \ref{conjecture:sampling-bias-UCB} and Lemma \ref{lem: three gap regimes} together suggest that the sample bias after CLT scaling satisfies
\begin{align}\label{eq:sample-bias-scaled}
    \sqrt{n^\star_{2,T}}(\mathbb E[\bar\mu_{2,T}]-\mu_2^T)=\begin{cases}
        -\Theta\left(\frac{1}{\Delta^T\sqrt{n^\star_{2,T}}}\right) \ \ \text{in the large gap regime}\\
        -\Theta\left(\frac{1}{\sqrt{\log n^\star_{2,T}}}\right) \ \ \text{in the moderate/small gap regime}.
    \end{cases}
\end{align}
Observe that while the sample bias vanishes to zero as the (typical) sample size grows to infinity, its convergence rate differs significantly under different parameter regimes. On one extreme, when the arm gap is a constant, arm 2's sample bias after CLT scaling in \eqref{eq:sample-bias-scaled} vanishes at a rate of $\Theta(\frac{1}{\sqrt{n^\star_{2,T}}})$. This coincides with the rate of convergence of a standard CLT in the Berry-Esseen theorem. In this regime, the challenge for estimating the mean reward of the inferior arm (arm 2) lies in data scarcity. Indeed, one expects to only get $n^\star_{2,T} = \Theta(\log T)$ data points from arm 2 after $T$ rounds, which incurs $\Theta(\frac{1}{\log T})$ negative bias according to \eqref{eq:sample-bias-scaled}. Arm 1, on the other hand, have nearly $T$ data points, and a negligibly small sample bias of $O(\frac{\log T}{T})$.


Compared with the constant gap regime, in the moderate/small gap regime, the magnitude of the sample bias after CLT scaling is significantly larger, which is $\Theta(\frac{1}{\sqrt{\log n^\star_{2,T}}})$ (see \eqref{eq:sample-bias-scaled}). In this regime, both arms receive $n^\star_{i,T} = \Theta(T)$ number of pulls. However, given how slowly $\frac{1}{\sqrt{\log n}}$ converges to zero as $n\rightarrow\infty$, the standard CLT-based statistical method to establish confidence interval for the arm's mean reward (cf. \cite{han2024ucb}) might suffer from a nontrivial error even in reasonably sized experiments, for both arms. 
In general, for arm gaps in between constant and moderate/small, the sample bias after CLT scaling interpolates between the two extreme cases.

In Appendix \ref{app:sampling_bias}, we conduct various numerical experiments and compare the simulation results with Conjecture \ref{conjecture:sampling-bias-UCB} for all three regimes in Figures \ref{fig:scaled-bias-small-gap}--\ref{fig:scaled-bias-moderate-gap}. The results show that as $T$ grows large, the empirical bias from the experiments converges to the conjectured value. A rigorous proof of Conjecture \ref{conjecture:sampling-bias-UCB} would require even higher-order analysis of the sample mean, which is beyond the scope of this paper, hence we leave it for further study.

\section{Conclusion}\label{sec: conclusion}
In this work, we prove a novel joint CLT of (1) the number of pulls of arms, and (2) the sample mean rewards of arms for data collected from a two-arm stochastic bandit under the UCB algorithms. This result leads to a number of interesting implications. First, it implies a non-standard CLT for the number of pulls and hence the pseudo-regret, revealing that both quantities experience large fluctuation in the small arm gap regimes. Second, it characterizes the correlation structure between the number of pulls and the sample mean rewards, leading to an explicit conjectured scale of sample bias, that is verified through numerical experiments. To achieve these results, we establish a novel perturbation analysis framework for characterizing dynamics of bandit systems driven by index-based algorithms beyond the fluid approximation, which are of independent interests. 

This work triggers a range of intriguing questions, opening up avenues for further exploration of sequential learning algorithms beyond the traditional lens of regret minimization. In particular, one direction is to utilize the high-level approaches developed in this work to characterize data collected from other/more complicated environment (e.g. contextual bandit, reinforcement learning), and under other algorithms, (e.g. Thompson Sampling). Another important next-question is to leverage the precise theoretical insights achieved here to improve the downstream data-driven statistical/operations tasks, through e.g. the design of better estimators/policies/mechanisms.

\section*{Acknowledgments}
We thank Dave Goldberg for a number of valuable comments that improved the paper.

\bibliographystyle{agsm}
\setstretch{0.9}


\newpage
\appendix
\onehalfspacing
\section{$K$-arm Extension}\label{app:k-arm}
We provide the $K$-arm extension of Theorem~\ref{thm:N1} in this section. Consider a $K$-arm bandit environment that generalizes the setup in Section~\ref{sec:prelim}. Namely, we have a sequence of bandit problems indexed by $T$. We adopt the same set of notation, only allowing $i \in \{1, \dots, K\}$ to incorporate more arms with $K \geq 3$. We assume WLOG that the arms are sorted, such that $\mu^T_i$ is decreasing in $i$.

Following the heuristic discussion in Section~\ref{sec: tech_overview}, we arrive at a system of linear equations \eqref{eq: linear-system-correlation-conjecture}, namely, 
\begin{align*}
        \begin{bmatrix}
1 & 1 & 1 & \dots & 1 \\
-I'_{1,2} & I'_{2,2} & 0 & \dots & 0 \\
\dots  & \dots  & \dots  & \dots & \dots  \\
-I'_{1,2} & 0 & 0 & \dots & I'_{K,2}
\end{bmatrix}
\begin{bmatrix}
\omega_1 \\ \omega_2 \\ \dots \\ \omega_K
\end{bmatrix}
=
\begin{bmatrix}
0 \\ I'_{1,1}\bar\varepsilon_1  - I'_{2,1}\bar\varepsilon_2 \\ \dots \\ I'_{1, 1}\bar\varepsilon_1 - I'_{K,1}\bar\varepsilon_K
\end{bmatrix},
\end{align*}
where we recall that $\bar\varepsilon_i = \bamu_{i, T} - \mu_i$ denotes the centered sample mean of arm $i = 1, \dots, K$, and $I'_{i, 1}$ and $I'_{i, 2}$ denote the partial derivatives of the index function $I$ w.r.t. the first and second argument, evaluated at $(\mu_i, n^{\star}_{i, T}, T)$ for each $i = 1, \dots, K$. Our theory approximates the true number of pulls $N_{i, T}$ by $n^{\star}_{i, T} + \omega_i$, where $(\omega_1, \dots, \omega_K)$ is the solution to the above linear systems. The following lemma characterizes the solution in closed-form.

\begin{lem}\label{lem: linear-system-correlation-conjecture}   
The solution to \eqref{eq: linear-system-correlation-conjecture} admits the following analytical form.
\begin{align*}
    & \omega_1 = -\left(1
+ \sum_{k=2}^{K}
\frac{I'_{1,2}}{I'_{k,2}}\right)^{-1}\sum_{k = 2}^K \frac{1}{I'_{k,2}}\left(I'_{1,1}\bar\varepsilon_1  - I'_{k,1}\bar\varepsilon_k\right) ,
    \\& \omega_i = \frac{I'_{1,1}\bar\varepsilon_1  - I'_{i,1}\bar\varepsilon_i}{I'_{i,2}} + \frac{I'_{1,2}}{ I'_{i,2}} \omega_1,\quad  i = 2, \dots, K.
\end{align*}
\end{lem}
We omit the proof. In the case of {\sf generalized UCB1}, namely $I(\mu, n, T) = \mu + \frac{f(T)}{\sqrt{n}}$, we have $I'_{i, 1}(\mu, n, T) = 1$ and $I'_{i, 2}(\mu, n, T) = -\frac{1}{2}n^{-\frac{3}{2}}f(T).$ Applying Lemma~\ref{lem: linear-system-correlation-conjecture} leads to:
\begin{corollary}\label{cor: linear-system-correlation-conjecture-f}
    In the case of {\sf generalized UCB1}, the solution to \eqref{eq: linear-system-correlation-conjecture} has the following form.
\begin{align*}
    & \omega_1 = \frac{2}{f(T)}\left(1
+ \sum_{k=2}^{K} \left(\frac{n^{\star}_{k, T}}{n^{\star}_{1, T}}\right)^{\frac{3}{2}}\right)^{-1}\sum_{k = 2}^K (n^{\star}_{k, T})^{\frac{3}{2}}\left(\bar\varepsilon_1  - \bar\varepsilon_k\right) ,
    \\& \omega_i = \frac{2}{f(T)}(n^{\star}_{i, T})^{\frac{3}{2}}\left(\bar\varepsilon_i  - \bar\varepsilon_1\right) + \left(\frac{n^{\star}_{i, T}}{n^{\star}_{1, T}}\right)^{\frac{3}{2}} \omega_1,\quad  i = 2, \dots, K.
\end{align*}
\end{corollary}
The fluid systems of equations analogous to \eqref{eq: fluid equation 2 arm f} in the general $K$ arm setting becomes
\begin{align}\label{eq: fluid equation K arm f}
    (n^{\star}_{i, T})^{-\frac{1}{2}} - (\nat)^{-\frac{1}{2}} = (f(T))^{-1}\Delta^T_i, \ \  \ \ i = 2, \dots, K; \ \ \ \  \sum_{i =1 }^T n^{\star}_{i, T} = T.
\end{align}
The mean reward gap $\Delta^T_i$ for each arm $i \in \{2, \dots, K\}$ may scale differently. The fluid scaling of $\bn^{\star}_{T}$ depends on the scaling regime of the arm gaps. Similar to the two-arm case, we introduce $\lx_{ij} \triangleq \lim_{T \to \infty} \frac{n^{\star}_{i, T}}{n^{\star}_{j, T}}$ to denote the fluid limit relative sampling ratio between arm $i$ and $j$ for any $i, j \in\{1, \dots, K\}.$ Corollary~\ref{cor: linear-system-correlation-conjecture-f} yields the following joint CLT in the K-arm setting.\\\\
\textbf{$K$-arm Joint CLT.\ }\label{conj: k-arm joint clt}
\textit{Consider a $K$-armed bandit environment that satisfies Assumption~\ref{assump: distribution}. The {\sf generalized UCB1} is implemented with $f(t)$ satisfying Assumption~\ref{assump: index}, with associated fluid approximations $\bn^{\star}_T$ for each $T \geq 1$ and the limiting sampling ratio $\lx_{i j} = \lim_{T \to \infty} \frac{n^{\star}_{i, T}}{n^{\star}_{j, T}}$ for each $i, j \in \{1, \dots, K\}$. Denote by $W_{i, T} = \frac{f(T)}{2 n^{\star}_{i \vee 2}}(N_{i, T} - n^{\star}_{i, T})$ and $Z_{i, T} = \sqrt{n^{\star}_{i, T}}\left(\bamu_{i, T} - \mu^T_i\right)$ for each $i = 1, \dots, K$. Then the $2 K$-dimensional random vector $(\boldsymbol{W}_T, \boldsymbol{Z}_T) = (W_{1, T}, \dots, W_{K, T}, Z_{1, T}, \dots Z_{K, T})$ satisfies
\begin{align*}
\begin{pmatrix}
    \boldsymbol{W}_T \\ 
    \boldsymbol{Z}_T
  \end{pmatrix} \ \xlongrightarrow[]{d} \ \mathcal{N}\!\left(
  \begin{pmatrix}
    \boldsymbol{0} \\ 
    \boldsymbol{0} 
  \end{pmatrix},
  \begin{pmatrix}
   \Sigma^1 & \Sigma^{1 2} \\
    (\Sigma^{1 2})^\top & \Sigma^2
  \end{pmatrix}
\right),
\end{align*}
where $\Sigma^1, \Sigma^2, \Sigma^{1 2} \in \RR^{K \times K}$. In particular, $\Sigma^2 = {\rm diag}\{\sigma_1^2, \dots, \sigma_K^2\}$. $\Sigma^{1 2}$ is given by
\begin{align*}
    & \Sigma^{1 2}_{1 1} = \left(\frac{\sum_{k = 2}^K\lx_{k 2}\sqrt{\lx_{k 1}}}{1 + \sum_{k = 2}^K (\lx_{k 1})^{\frac{3}{2}}}\right)\sigma^2_1, \ \ \Sigma^{1 2}_{1 i} = -\frac{\lx_{i 2}}{1 + \sum_{k = 2}^K (\lx_{k 1})^{\frac{3}{2}}} \sigma_i^2,\\
    & \Sigma^{1 2}_{i j} = \left(\mathbbm{1}_{\{j = i\}} -\frac{\lx_{j 1}\sqrt{\lx_{i 1}}}{1 + \sum_{k = 2}^K (\lx_{k 1})^{\frac{3}{2}}}\right)\sigma_j^2,
\end{align*}
for any $i \in \{2, \dots, K\}$ and $j \in \{1, \dots, K\}.$ $\Sigma^1$ is given by
\begin{align*}
    &\Sigma^1_{1 1} = \left(\frac{\sum_{k = 2}^K\lx_{k 2}\sqrt{\lx_{k 1}}}{1 + \sum_{k = 2}^K (\lx_{k 1})^{\frac{3}{2}}}\right)^2\sigma^2_1 + \sum_{l = 2}^K \left(\frac{\lx_{l 2}}{1 + \sum_{k = 2}^K (\lx_{k 1})^{\frac{3}{2}}} \right)^2\sigma_l^2,\\ 
    & \Sigma^{1}_{1 i} = \Sigma^{1}_{i 1} = -\frac{\sqrt{\lx_{i 1}}\sum_{k = 2}^K\lx_{k 2}\sqrt{\lx_{k 1}}}{\left(1 + \sum_{k = 2}^K (\lx_{k 1})^{\frac{3}{2}}\right)^2}\sigma_1^2 - \sum_{l = 2}^K\left(\frac{\lx_{l 2}}{1 + \sum_{k = 2}^K (\lx_{k 1})^{\frac{3}{2}}}\right)\left(\mathbbm{1}_{\{l = i\}} -\frac{\lx_{l 1}\sqrt{\lx_{i 1}}}{1 + \sum_{k = 2}^K (\lx_{k 1})^{\frac{3}{2}}}\right) \sigma_l^2,\\
    & \Sigma^1_{i j} = \sum_{l = 1}^K  \left(\mathbbm{1}_{\{l = i\}} -\frac{\lx_{l 1}\sqrt{\lx_{i 1}}}{1 + \sum_{k = 2}^K (\lx_{k 1})^{\frac{3}{2}}}\right)\left(\mathbbm{1}_{\{l = j\}} -\frac{\lx_{l 1}\sqrt{\lx_{j 1}}}{1 + \sum_{k = 2}^K (\lx_{k 1})^{\frac{3}{2}}}\right) \sigma_l^2,
\end{align*}
for any $i, j \in \{2, \dots, K\}$.}\\\\
The $K$-arm joint CLT has a complicated form that depends on specific scaling rates of $\Delta^T_i, i = \{2, \dots, K\}$. The proof is expected to largely follow the similar route taken in the two-arm setting. Technically, the reduction to the two-arm setting is fairly straightforward in certain gap regimes, for example, (1) the case of separated superior arm, namely, when all inferior arms are in the large-gap regime, and (2) the case of indistinguishable arms, where all inferior arms are in the small-gap or the moderate-gap regime. We omit the proof for brevity, and leave a complete proof in arbitrary arm-gap regime for future investigation.

In what follows, we focus on two special cases, where the form of the joint CLT is drastically simplified.




\paragraph{Separated superior arm.}
Consider the case that the superior arm is clearly separated from inferior arms. More precisely, $\Delta^T_i \geq \epsilon > 0$ for any $i \in \{2, \dots, K\}$ and all $T \geq 1.$ In this case, the fluid approximation from \eqref{eq: fluid equation K arm f} has the following scaling: $n^{\star}_{1, T} \sim T$, $n^{\star}_{i, T} \sim \left(\frac{f(T)}{\Delta_i}\right)^2$  for $i = 2, \dots, K.$ Consequently $\lx_{i 1} = 0$ and $\lx_{i 2} = \left(\frac{\Delta_i}{\Delta_2}\right)^2$ for $ i = 2, \dots, K.$ In the case of {\sf UCB1}, we have $f(t) = \sqrt{2 \log t}$, and the joint CLT can be specified as 
\begin{align*}
\begin{pmatrix}
    \frac{\Delta_2^2}{2 \sqrt{2 \log T}}N_{2, T} - \frac{\sqrt{2 \log T}}{2} \\
    \dots\\
    \frac{\Delta_K^2}{2 \sqrt{2 \log T}}N_{K, T} - \frac{\sqrt{2 \log T}}{2} \\
    \frac{\sqrt{2 \log T}}{\Delta_2}\left(\bamu_{2, T} - \mu^T_2\right)\\
    \dots \\
    \frac{\sqrt{2 \log T}}{\Delta_K}\left(\bamu_{K, T} - \mu^T_K\right)\\
  \end{pmatrix} \ \xlongrightarrow[]{d} \ \mathcal{N}\!\left(
  \begin{pmatrix}
    \boldsymbol{0} \\ 
    \boldsymbol{0} 
  \end{pmatrix},
  \begin{pmatrix}
     \Sigma^{\star} & \Sigma^{\star}\\
    \Sigma^{\star} & \Sigma^{\star}\\ 
  \end{pmatrix}
\right),
\end{align*}
where $\Sigma^{\star} = \textrm{diag}\{\sigma_2^2, \dots, \sigma_K^2\}$. The first arm's number of pull is determined then by $N_{1, T} = T - \sum_{i = 2}^K N_{i,T}.$ In other words, in this case, all inferior arms $ i \in \{2, \dots, K\}$ become asymptotically uncorrelated. In particular, the $i^{\rm th}$ arm's number of pull and its centered sample mean satisfy
\begin{align}\label{eq: separate correlation between number of pull and sample mean}
    \frac{\Delta_i^2}{2 \sqrt{2 \log T}}N_{i, T} - \frac{\sqrt{2 \log T}}{2} - \frac{\sqrt{2 \log T}}{\Delta_i}\left(\bamu_{i, T} - \mu^T_i\right) \xlongrightarrow[]{p} 0, 
\end{align}
and the number of pulls satisfies the (marginal) CLT
\begin{align}\label{eq: separate number of pull CLT}
   \frac{\Delta_i^2}{2 \sqrt{2 \log T}}N_{i, T} - \frac{\sqrt{2 \log T}}{2}  \xlongrightarrow[]{d} \mathcal{N}\left(0, \sigma_i^2\right).
\end{align}
Once again, the CLT \eqref{eq: separate number of pull CLT} of the number of pulls of inferior arms recovers that of Theorem 7 in \cite{fan2022typical}. This serves as a sanity check for the $K$-arm joint CLT. On the other hand, the correlation structure characterized by \eqref{eq: separate correlation between number of pull and sample mean} is novel to the literature. As was mentioned above, the joint CLT in this setting can be proved following a reduction to the two-arm setting. \cite{fan2022typical} pointed out why such a reduction is possible: \textit{So, effectively, each inferior arm only competes with the superior arm to be played, and the analysis in multi-armed settings reduces to that in the two-armed setting.}

\paragraph{Indistinguishable arms.} Another special case is when all inferior arms are indistinguishable, namely all $\Delta^T_{i}$ are in the small-gap regime. For simplicity, we assume all arms have identical mean reward, $\mu^T_i = \mu$. Thus $\Delta^T_i = 0$ for all $T \geq 1$ and $i \in \{1, \dots, K\}$. In this case, $n^{\star}_{i, T} = \frac{T}{K}$ and $\lx_{i j} = 1$ for all $ i, j \in \{1, \dots, K\}.$ The $K$-arm joint CLT implies that, for each arm $i \in \{1, \dots, K\}$,
\begin{align}\label{eq: identical correlation between number of pull and sample mean}
    \frac{K f(T)}{2 T}N_{i, T} - \frac{f(T)}{2} - \sqrt{\frac{T}{K}}\sum_{k = 1}^K \left( - \frac{1}{K} + \mathbbm{1}_{\{k = i\}}\right)\left(\mu_{k, T} - \mu\right) \xlongrightarrow[]{p} 0,
\end{align}
and the number of pulls satisfies the (marginal) CLT
\begin{align}\label{eq: identical number of pull CLT}
  \frac{K f(T)}{2 T}N_{i, T} - \frac{f(T)}{2}  \xlongrightarrow[]{d} \mathcal{N}\left(0\ ,\ \  \frac{1}{K^2}\sum_{\substack{j = 1,\ j\neq i}}^K \sigma_j^2 + \left(1 - \frac{1}{K}\right)^2 \sigma_i^2 \right).
\end{align}

\section{Proof of Theorem \ref{thm:N1}}\label{app: proof of main thm}
\subsection{Helper lemmas}
We first state some  helper lemmas.
\begin{lemma}\label{lem: sub_gaussian}
Let \(Y_1, Y_2, \dots, Y_n\) be independent $\sigma$-sub-Gaussian random variables, then $Y_1 + Y_2 + \dots + Y_n$ is $\left(\sigma\sqrt{n}\right)$-sub-Gaussian. 
\end{lemma}

\begin{lemma}[Lyapunov CLT for triangular arrays]\label{lem: triangular CLT}
Let \( \{ Y_{n,i} : 1 \leq i \leq n \} \) be a triangular array, where \( Y_{n,1}, Y_{n,2}, \dots, Y_{n,n} \) are independent for each \( n \), with $\E[Y_{n,i}] = 0$ and $\text{Var}(Y_{n,i}) = \sigma_{n,i}^2$ for $1 \le i \le n$. The total variance satisfies $ \sum_{i=1}^n \sigma_{n,i}^2 = \sigma_n^2 \quad \text{with} \quad \sigma_n^2 \to \sigma^2 \quad \text{as} \quad n \to \infty.$ Furthermore, there exists a constant $\delta > 0$ such that the Lyapunov condition is satisfied:
\[
    \frac{1}{\sigma_n^2} \sum_{i=1}^n \mathbb{E}[|Y_{n,i}|^{2+\delta}] \to 0 \quad \text{as} \quad n \to \infty.
\]
Then, the normalized sum converges in distribution to a standard normal distribution:
\[
\frac{1}{\sigma_n} \sum_{i=1}^n Y_{n,i} \xrightarrow{d} N(0, 1) \quad \text{as} \quad n \to \infty.
\]
\end{lemma}

\begin{lemma}[Etemadi's inequality]\label{lm: etemadi}
    Let \(Y_1, Y_2, \dots, Y_n\) be independent random variables and define the partial sums $S_k = \sum_{i=1}^k X_i, \  1 \le k \le n.$
Then, for every \(\epsilon > 0\), we have
\[
\PP\Biggl(\max_{1 \le k \le n} \bigl| S_k \bigr| \ge 3\epsilon\Biggr) \le 3 \max_{1 \le k \le n} \PP\Bigl(\bigl| S_k \bigr| \ge \epsilon\Bigr).
\]
\end{lemma}

\begin{lemma}[Slutsky's theorem]\label{lem:slutsky}
    Let $X_n, Y_n$ be sequence of scalar/vector/matrix random elements. If $X_n$ converges in distribution to a random element $X$ and $Y_n$ converges in probability to a constant $c$, then $    X_n + Y_n \xlongrightarrow[]{d} X + c; \ \ X_n Y_n \xlongrightarrow[]{d} X c.$
\end{lemma}

\begin{lemma}[Lemma 1 in \cite{jamieson2014lil}]\label{lem: anytime-LIL}
    Let \(Y_1, Y_2, \dots\) be i.i.d. centered $\sigma$-sub-Gaussian random variables, and define $S_t = \sum_{i=1}^{t} Y_i.$ Then, for each \(\theta \in (0,1)\) and $\delta > 0 $, we have
\[
\mathbb{P}\Biggl(\exists\, t\ge 1 :\, S_t \ge (1 + \sqrt{\theta})\sigma \sqrt{2(1 + \theta)\,t\,\log\!\left(\frac{\log ((1 + \theta)t + 2)}{\delta}\right)}\Biggr) \le \frac{2 + \theta}{\theta}\left(\frac{\delta}{\log(1 + \theta)}\right)^{1 + \theta}.
\]
\end{lemma}

\begin{lemma}\label{lem: maximal inequality}
Suppose $Y_i, i \geq 1$ are i.i.d. centered $\sigma$-sub-Gaussian. Then for any $1 \leq s_1 < s_2$, we have
\[
\PP\left(\max_{s_1 \leq u < v \leq s_2} \left| \frac{\sum_{i = 1}^u Y_i}{u} - \frac{\sum_{i = 1}^v Y_i}{v}\right| >  a \right) \leq 8\exp\left(-\frac{a^2 s_1^2}{72 \sigma^2 (s_2 - s_1)} \right).
\]    
\end{lemma}

\begin{lemma}[$\bn_T$ diverges] \label{lem: nstar_diverge}
Under Assumption \ref{assump: distribution}, it holds true that $\nat \geq \nbt$, and both $\nat, \nbt$ diverge as $T \to \infty$, where $\bn^{\star}_T$ is the solution to \eqref{eq: fluid equation 2 arm f}. 
\end{lemma}
Here, Lemma~\ref{lem: sub_gaussian} - Lemma~\ref{lem:slutsky} are classical probability results. Lemma~\ref{lem: anytime-LIL} is a finite-time non-asymptotic law of iterated logarithm quoted from \cite{jamieson2014lil}. Lemma~\ref{lem: maximal inequality} is a maximal inequality, whose proof is provided in \Cref{app: other proof}. Lemma \ref{lem: nstar_diverge} follows immediately from Lemma \ref{lem: three gap regimes}, the proof of which we omit.

\subsection{Proof of Theorem \ref{thm:N1}}
We first state a crucial intermediate result. For any sequence $x_T, T \geq 1$, let's denote by $n^x_{i, T}\triangleq (1 - x_T)n^{\star}_{i, T}$ for $i = 1, 2$.
\begin{lemma}\label{lem: high-prob-N1-n1}
Suppose $\delta_T, T \geq 1$ is an arbitrary sequence satisfying $\delta_T = o(1)$ and $\delta_T \leq \frac{1}{2}$ for all $T$. Then under the conditions of Theorem \ref{thm:N1},
\[
f(T) \frac{N_{2, T} - \nbt}{\nbt} - \frac{- \bar\mu_1(\nad) + \bar\mu_2(\nbd) + \Delta^T_2}{\left(\frac{1}{2}(\nat)^{-\frac{3}{2}} + \frac{1}{2}(\nbt)^{-\frac{3}{2}} \right)\nbt} \xlongrightarrow[]{p} 0.  
\]
\end{lemma}
The following asymptotic characterization of the term appearing in the statement of Lemma~\ref{lem: high-prob-N1-n1} follows directly from the triangular array CLT (Lemma \ref{lem: triangular CLT}).
\begin{lemma}\label{lem: bar-mu-clt}
    Let $M^\delta_T \triangleq \frac{- \bar\mu_1(\nad) + \bar\mu_2(\nbd) + \Delta^T_2}{\left(\frac{1}{2}(\nat)^{-\frac{3}{2}} + \frac{1}{2}(\nbt)^{-\frac{3}{2}} \right)\nbt}$. Then $M^{\delta}_T \xlongrightarrow[]{d} \mathcal{N}\left(0, \frac{4\lx\sigma_1^2 + 4 \sigma_2^2}{\left(1 + (\lx)^{\frac{3}{2}}\right)^2}\right).$
\end{lemma}
Lemma~\ref{lem: high-prob-N1-n1} and Lemma~\ref{lem: bar-mu-clt} allow us to recover the weak LLN of the number of pulls that appear in prior work. Furthermore, they imply a loose high probability bound on the convergence rate of the LLN.
\begin{corollary}\label{cor: N1-concentrate}
    $\bN_T$ satisfies the weak law of large number $\frac{N_{i, T}}{n^{\star}_{i, T}} \xlongrightarrow[]{p} 1, \ \ i = 1, 2.$ Furthermore, for an arbitary sequence $\delta_T$ satisfying $\delta_T = o(1)$ 
    \begin{align*}
        &\lim_{T \to \infty} \PP\left( - \frac{1}{\delta_T f(T)} \leq \frac{N_{i, T} -n^{\star}_{i, T}}{\nbt} \leq   \frac{1}{\delta_T f(T)}\right) = 1,
    \end{align*}
\end{corollary}
The proof of Lemma \ref{lem: high-prob-N1-n1} is provided in \Cref{sec: proof of key lemma}. The proof of Lemma~\ref{lem: bar-mu-clt} can be found in \Cref{app: other proof}. We now leverage these results to complete the proof of Theorem \ref{thm:N1}.

\begin{proof}[Proof of Theorem \ref{thm:N1}]
We begin by introducing some additional notation to simplify the exposition. Denote by $Z^{\star}_{i, T} \triangleq \sqrt{n^{\star}_{i, t}}\left(\bamu_i(n^{\star}_{i, T}) - \mu^T_i\right)$ and $Z_{i, T} \triangleq \sqrt{n^{\star}_{i, t}}\left(\bamu_i(N_{i, T}) - \mu^T_i\right)$ for $i = 1, 2$. We prove the following weak convergence.
\begin{align}
    Z_{2, T} - Z^{\star}_{2, T} &\xlongrightarrow[]{p} 0, \label{eq: n to N in prob Z1}\\
    \sqrt{\frac{\nat}{\nbt}}\left(Z_{1, T} - Z^{\star}_{1, T} \right) &\xlongrightarrow[]{p} 0.\label{eq: n to N in prob Z2}
\end{align}
Applying Corollary~\ref{cor: N1-concentrate}, the following events happen with probability approaching $1$ as $T \to \infty$:
 \begin{align}\label{eq: N-n bound}
     \left|N_{i, T} - n^{\star}_{i, T}\right| \leq \frac{\nbt}{\sqrt{f(T)}}, i = 1, 2 
 \end{align}
 Assuming (\ref{eq: N-n bound}), we have for fixed $\epsilon > 0$
\begin{align}
     &\left\{\left|Z_{2, T} - Z^{\star}_{2, T} \right| > \epsilon \right\} = \left\{\left|\sqrt{\nbt}\left(\bamu_2(\nbt) - \bamu_{2, T}\right) \right| > \epsilon \right\}, \NNN\\
     &\ \ \subseteq \left\{\max_{n^{\star}_{2, T}-\frac{1}{\sqrt{f(T)}}\nbt \leq u,v \leq n^{\star}_{2, T} + \frac{1}{\sqrt{f(T)}}\nbt}\left|- \bar\mu_2(u) + \bar\mu_2(v)\right| > (\nbt)^{-\frac{1}{2}}\epsilon \right\}. \label{eq: n to N in prob linking to maximal 1}
 \end{align}
Applying Lemma \ref{lem: maximal inequality}, we have 
 \begin{align*}
     \PP\left((\ref{eq: n to N in prob linking to maximal 1})\right) \leq 8 \exp\left(- \frac{\left(1 - \frac{1}{\sqrt{f(T)}}\right)^2(\nbt)^{2}(\nat)^{-1}\epsilon^2}{72 \sigma^2 \frac{2}{\sqrt{f(T)}}\nbt}\right) = \exp\left(- O\left(\sqrt{f(T)}\right)\right),
 \end{align*}
 which vanishes as $T \to \infty$ since $f(T) = \omega(1)$. This concludes the proof of eq. (\ref{eq: n to N in prob Z1}). Similarly, for (\ref{eq: n to N in prob Z2}), assuming (\ref{eq: N-n bound}), we have for fixed $\epsilon > 0$
 \begin{align}
     &\left\{\sqrt{\frac{\nat}{\nbt}}\left|Z_{1, T} - Z^{\star}_{1, T} \right| > \epsilon \right\} = \left\{\left|\bamu_1(\nat) - \bamu_{1, T} \right| > (\nbt)^{\frac{1}{2}}(\nat)^{-1}\epsilon \right\}, \NNN\\
     &\ \ \subseteq \left\{\max_{n^{\star}_{1, T}-\frac{1}{\sqrt{f(T)}}\nbt \leq u,v \leq n^{\star}_{1, T} + \frac{1}{\sqrt{f(T)}}\nbt}\left|- \bar\mu_1(u) + \bar\mu_1(v)\right| > (\nbt)^{\frac{1}{2}}(\nat)^{-1}\epsilon \right\}. \label{eq: n to N in prob linking to maximal 2}
 \end{align}
 Applying Lemma~\ref{lem: maximal inequality}, we have
 \[
 \PP\left((\ref{eq: n to N in prob linking to maximal 2})\right) \leq 8 \exp\left(- \frac{\left(1 - \frac{1}{\sqrt{f(T)}}\right)^2(\nat)^2\nbt (\nat)^{-2}\epsilon^2}{72 \sigma^2 \frac{2}{\sqrt{f(T)}}\nbt}\right) = \exp\left(- O\left(\sqrt{f(T)}\right)\right),
 \]
 which also vanishes. The above concludes the proof of eq. (\ref{eq: n to N in prob Z1}) and (\ref{eq: n to N in prob Z2}). Since $\nat \geq \nbt$ (by Lemma~\ref{lem: nstar_diverge}), we note that (\ref{eq: n to N in prob Z2}) also implies $Z_{1, T} - Z^{\star}_{1, T} \xlongrightarrow[]{p} 0.$  By Lemma~\ref{lem: triangular CLT} and the fact that $\nat, \nbt$ diverges as $T \to \infty$ (Lemma~\ref{lem: nstar_diverge}), we have the following CLT for $Z^{\star}_{i, T}, i = \{1, 2\}$.
 \begin{align}
     & Z^{\star}_{1, T} \xlongrightarrow[]{d}  \mathcal{N}\left(0, \sigma_1^2\right), \label{eq: Z1star_CLT}\\
     &Z^{\star}_{2, T} \xlongrightarrow[]{d}   \mathcal{N}\left(0, \sigma_2^2\right), \label{eq: Z2star_CLT}
 \end{align}
 where we remark that Slutsky's theorem (Lemma~\ref{lem:slutsky}) is used with $\lim_{T \to \infty} \sigma^T_i = \sigma_i$ for $i = 1, 2$, according to Assumption \ref{assump: distribution}. Combining the above and applying Slutsky's theorem again, we derive the CLT for $Z_{1, T}$ and $Z_{2, T}$:
 \begin{align}
     & Z_{1, T} \xlongrightarrow[]{d}  \mathcal{N}\left(0, \sigma_1^2\right), \label{eq: Z1_CLT}\\
     &Z_{2, T} \xlongrightarrow[]{d}   \mathcal{N}\left(0, \sigma_2^2\right), \label{eq: Z2_CLT}
 \end{align}
 
 On the other hand, Lemma \ref{lem: high-prob-N1-n1} yields (with $\delta_T \equiv 0$)
 \begin{align*}
     f(T)\frac{N_{2, T} - \nbt}{\nbt} - \frac{- \bar\mu_1(\nat) + \bar\mu_2(\nbt) + \Delta^T}{\left(\frac{1}{2}(\nat)^{-\frac{3}{2}} + \frac{1}{2}(\nbt)^{-\frac{3}{2}} \right)\nbt} \xlongrightarrow[]{p} 0,
 \end{align*}
 or equivalently, 
  \begin{align}\label{eq: N and Z-star-raw}
     f(T)\frac{N_{2, T} - \nbt}{\nbt} - \frac{  Z^{\star}_{2, T}}{\frac{1}{2}\left(\frac{\nbt}{\nat}\right)^{\frac{3}{2}} + \frac{1}{2}} + \frac{  Z^{\star}_{1, T}}{\frac{1}{2}\left(\frac{\nat}{\nbt}\right)^{\frac{1}{2}} + \frac{1}{2}\frac{\nbt}{\nat}}\xlongrightarrow[]{p} 0.
 \end{align}
Recall that $\lx = \lim_{T \to \infty}\frac{\nbt}{\nat}$, which further implies
\[
 \lim_{T \to \infty}\frac{\frac{1}{2} (\lx)^{\frac{3}{2}} + \frac{1}{2}}{\frac{1}{2}\left(\frac{\nbt}{\nat}\right)^{\frac{3}{2}} + \frac{1}{2}} = 1\ , \ \lim_{T \to \infty}\frac{\frac{1}{2} (\lx)^{\frac{3}{2}} + \frac{1}{2}}{\left(\frac{1}{2}\left(\frac{\nat}{\nbt}\right)^{\frac{1}{2}} + \frac{1}{2}\frac{\nbt}{\nat}\right)} = \sqrt{\lx} \in [0, 1].
 \]
 We multiply the LHS of (\ref{eq: N and Z-star-raw}) by a factor of $\frac{1 + \left(\lx\right)^{\frac{3}{2}}}{2} (\in [\frac{1}{2}, 1])$, and denote by $W_2 = \frac{1 + (\lx)^{\frac{3}{2}}}{2}\frac{f(T)}{\nbt}\left(N_{2, T} - \nbt\right)$.  The fact that $Z^{\star}_{i, T}$ both converge in distribution (see \eqref{eq: Z1star_CLT} and \eqref{eq: Z2star_CLT}), combined with a use of Slutsky's theorem allow us to conclude from (\ref{eq: N and Z-star-raw}) that 
 \begin{align}\label{eq: N and Z-star}
     W_2 -   Z^{\star}_{2, T} + \sqrt{\lx}Z^{\star}_{1, T} \xlongrightarrow[]{p} 0,
 \end{align}
Combining \eqref{eq: N and Z-star} with \eqref{eq: Z1star_CLT} and \eqref{eq: Z2star_CLT}, and noticing that $Z^{\star}_{1, T}$ and $Z^{\star}_{2, T}$ are independent , we have  
 \[
 \left( W_2 -   Z^{\star}_{2, T} + \sqrt{\lx}Z^{\star}_{1, T}, Z^{\star}_{1, T}, Z^{\star}_{2, T}\right) \xlongrightarrow[]{d} \left(0, Z_1, Z_2\right),
 \]
 with $(Z_1, Z_2)$  following $ \mathcal{N}\!\left(
  \begin{pmatrix}
    0 \\ 
    0 
  \end{pmatrix},
  \begin{pmatrix}
   \sigma_1^2  & 0 \\
    0 & \sigma_2^2 
  \end{pmatrix}
\right).$ This further implies 
 \[
 \left( W_2, Z^{\star}_{1, T}, Z^{\star}_{2, T}\right) \xlongrightarrow[]{d} \left(Z_2 - \sqrt{\lx}Z_1, Z_1, Z_2\right).
 \]
 Combining the above with (\ref{eq: n to N in prob Z1}) and (\ref{eq: n to N in prob Z2}) and applying Slusky's Theorem again, we finally conclude that 
  \[
 \left( W_2, Z_{1, T}, Z_{2, T}\right) \xlongrightarrow[]{d} \left(Z_2 - \sqrt{\lx}Z_1, Z_1, Z_2\right),
 \]
 which is the desired result of Theorem \ref{thm:N1}. We thus complete the proof. 
\end{proof}

\subsection{Proof of Lemma~\ref{lem: high-prob-N1-n1}}\label{sec: proof of key lemma}
We first introduce notation $\wa$  to denote 
\begin{align}
    \wa \triangleq \nbt\left(1 +  \frac{\frac{- \bar\mu_1(\nad) + \bar\mu_2(\nbd) + \Delta^T}{\left(\frac{1}{2}(\nat)^{-\frac{3}{2}} + \frac{1}{2}(\nbt)^{-\frac{3}{2}} \right)\nbt} + \epsilon}{f(T)} \right). \label{eq: wat-def}
\end{align}
Before starting to prove the lemma, let's first make an observation that the following  event, 
\begin{align}
\left|\frac{- \bar\mu_1(\nad) + \bar\mu_2(\nbd) + \Delta^T}{\left(\frac{1}{2}(\nat)^{-\frac{3}{2}} + \frac{1}{2}(\nbt)^{-\frac{3}{2}} \right)\nbt} \right| \leq \sqrt{\frac{f(T)}{32\log f(T)}}, \label{eq: bounds on Z}
\end{align}
occurring with probability at least $1 - 2\exp\left(-\frac{f(T)}{1024 \sigma^2 \log f(T)}\right)$, which vanishes as $T \to \infty$. We defer the proof to \Cref{app: other proof}. (see Lemma \ref{lem: wat-close-to-nat}) We shall assume that (\ref{eq: bounds on Z}) holds throughout the proof.

\begin{proof}[Proof of Lemma \ref{lem: high-prob-N1-n1}]
It suffices to prove that for any fixed $\epsilon>0$, $\PP\left(N_{2, T} > \wa\right)$ and $\PP\left(N_{2, T} < \wb\right)$, or equivalently,  $\PP\left(N_{1, T} \geq T -  \wb\right)$, both vanish as $T\rightarrow\infty$. The nature of the {generalized \sf UCB1} algorithm (Algorithm \ref{algo: general-ucb1}) implies that for any $a < T$ and $i \in \{1, 2\}$
\begin{align}
    \{N_{i, T} > a \} &\subseteq \left\{\exists t \leq T-1, N_{i, t} = a, I_{i, t+1} > I_{-i, t+1} \right\}, \NNN\\
    & = \left\{\exists t \leq T-1, N_{i, t} = a, \bamu_i(a)+ \frac{f(t+1)}{\sqrt{a}} > \bamu_{-i}(t-a) + \frac{f(t+1)}{\sqrt{t - a}} \right\}, \NNN\\
    & \subseteq \left\{\exists t \leq T-1, \bamu_i(a) - \bamu_{-i}(t-a)  > \frac{f(t+1)}{\sqrt{t - a}} - \frac{f(t+1)}{\sqrt{a}} \right\}. \label{eq: index_bound_1}
\end{align}
We introduce an auxiliary threshold $\tau^1_T = T - \eta_T n^\star_{2, T}, \tau^2_T = T - \eta_T n^\star_{1, T}$, where 
\[
\eta_T = \frac{100 \sigma}{1 - 2\beta}\frac{\sqrt{\log(2 + \log(2 +f(\sqrt{\nbt})))}}{f(\sqrt{\nbt})},
\]
with $0 < \beta < \frac{1}{2}$ and $\sigma > 0$ defined in Assumption \ref{assump: distribution}, is a vanishing sequence as $T \to \infty$ (since $\nbt \to \infty$ as $T \to \infty$ by Lemma \ref{lem: nstar_diverge}). We shall further decompose eq. (\ref{eq: index_bound_1}) according to whether $\tau^i_T \leq t \leq T-1$ or $t < \tau^i_T$,
\begin{align}
    (\ref{eq: index_bound_1}) & \ \subseteq \left\{\exists t\in  [\tau^i_T , T-1],\  \bamu_i(a) - \bamu_{-i}(t-a)  > \frac{f(t+1)}{\sqrt{t - a}} - \frac{f(t+1)}{\sqrt{a}}  \right\}\label{eq:index_bound_1_near_T}\\
    & \quad\quad \bigcup \left\{\exists t < \tau^i_T,\  \bamu_i(a) - \bamu_{-i}(t-a)  > \frac{f(t+1)}{\sqrt{t - a}} - \frac{f(t+1)}{\sqrt{a}}.  \right\}\label{eq: index_bound_1_far}
\end{align}
In the sequel, we obtain probability bounds of events (\ref{eq:index_bound_1_near_T}) and (\ref{eq: index_bound_1_far}) for $i = 2, a = \wa$ and $i = 1, a = T - \wb$ respectively. 
\paragraph{Case 1. Treating (\ref{eq:index_bound_1_near_T})}
\begin{align}
    (\ref{eq:index_bound_1_near_T}) & = \left\{\exists t\in  [\tau^i_T , T-1],\  \bamu_i(a) - \bamu_{-i}(t-a)  > \frac{f(t+1)}{\sqrt{t - a}} - \frac{f(t+1)}{\sqrt{a}}  \right\}\NNN\\
    & \ \ \subseteq \left\{\exists t\in  [\tau^i_T , T-1],\  \bamu_i(a) - \bamu_{-i}(t-a)  > \frac{f(T)}{\sqrt{T - a}} - \frac{f(T)}{\sqrt{a}} \right\},\label{eq: index_bound_1_near_T_1}
\end{align}
which follows from the monotonicity of $f(t)/\sqrt{t - a}$ (decreasing) and $f(t)$ (increasing), by Assumption \ref{assump: index} on $f$. For $i = 2, a = \wa$, the RHS in eq. (\ref{eq: index_bound_1_near_T_1}) is $\frac{f(T)}{\sqrt{T - \wa}} - \frac{f(T)}{\sqrt{\wa}}$, and for $i = 1, a = T - \wb$, the RHS is $\frac{f(T)}{\sqrt{\wb}} - \frac{f(T)}{\sqrt{T - \wb}}$. In both cases, we expand the RHS in eq. (\ref{eq: index_bound_1_near_T_1}) at $\nbt$ (replacing $\wa$ / $\wb$), and get
\begin{align}
    &\frac{f(T)}{\sqrt{T - \wa}} - \frac{f(T)}{\sqrt{\wa}} \NNN\\
    &\qquad= \frac{f(T)}{\sqrt{T - \nbt}}\left(1 - \frac{1}{2}\left(\frac{T - \wa}{T - \nbt} - 1\right) + \xi_1\left(\frac{T - \wa}{T - \nbt} - 1\right)^2 \right) \NNN\\
    &\qquad\qquad\qquad\qquad  - \frac{f(T)}{\sqrt{\nbt}}\left( 1- \frac{1}{2}\left(\frac{\wa}{\nbt} - 1\right) + \xi_2\left(\frac{\wa}{\nbt} - 1\right)^2 \right), \NNN\\
    & \qquad = \frac{f(T)}{\sqrt{\nat}} - \frac{f(T)}{\sqrt{\nbt}}  + f(T)\left(\frac{1}{2}(\nat)^{-\frac{3}{2}} + \frac{1}{2}(\nbt)^{-\frac{3}{2}}\right)\left(\wa - \nbt\right)\NNN\\
    &\qquad\qquad\qquad\qquad  + f(T)\left(\xi_1(\nat)^{-\frac{5}{2}} - \xi_2(\nbt)^{-\frac{5}{2}}\right) \left(\wa - \nbt\right)^2, \NNN\\
    & \qquad = \nbt\left(\frac{1}{2}(\nat)^{-\frac{3}{2}} + \frac{1}{2}(\nbt)^{-\frac{3}{2}}\right) \epsilon -\bamu_1(\nad) + \bamu_2(\nbd)  \label{eq: index_bound_1_near_T_1_RHS_wa_main} \\
    &\qquad\qquad\qquad\qquad  + f(T)\left(\xi_1(\nat)^{-\frac{5}{2}} - \xi_2(\nbt)^{-\frac{5}{2}}\right) \left(\wa - \nbt\right)^2 \label{eq: index_bound_1_near_T_1_RHS_wa_res},
\end{align}
and (similarly)
\begin{align}
    &\frac{f(T)}{\sqrt{\wb}} - \frac{f(T)}{\sqrt{T - \wb}}  \NNN\\
    & \qquad = \nbt\left(\frac{1}{2}(\nat)^{-\frac{3}{2}} + \frac{1}{2}(\nbt)^{-\frac{3}{2}}\right) \epsilon +\bamu_1(\nad) - \bamu_2(\nbd) \label{eq: index_bound_1_near_T_1_RHS_wb_main} \\
    &\qquad\qquad\qquad\qquad  - f(T)\left(\xi_3(\nat)^{-\frac{5}{2}} - \xi_4(\nbt)^{-\frac{5}{2}}\right) \left(\wb - \nbt\right)^2 \label{eq: index_bound_1_near_T_1_RHS_wb_res},
\end{align}
where we use the fact that $\bn^{\star}_T$ solves the fluid fixed-point equations (\ref{eq: fluid system}), explicitly
\[
\frac{f(T)}{\sqrt{\nbt}} - \frac{f(T)}{\sqrt{\nat}} = \Delta^T; \quad \nat+ \nbt = T,
\]
and $\xi_i, i = 1, \dots, 4$ are constants derived from the expansion  $ \frac{1}{\sqrt{x}} = 1 - \frac{1}{2}(x - 1) + \xi(x - 1)^2$. In our case, we take $T$ large enough such that  $\frac{1}{\log f(T)} < \epsilon < \frac{f(T)}{32}$ (this is possible as $\epsilon$ is a constant and $f(T) \to \infty$ with $T$ by Assumption \ref{assump: index}). Recall also that we assume (\ref{eq: bounds on Z}). Then it follows that 
\[
\left|\frac{\wa - \nbt}{\nbt}\right|, \left|\frac{\wa - \nbt}{\nat}\right|, \left|\frac{\wb - \nbt}{\nbt}\right|, \left|\frac{\wb - \nbt}{\nat}\right| \leq \frac{1}{2}
\]
and we derive bounds $|\xi_i | \leq 1, i= 1, \dots 4$, since  $|\frac{1}{\sqrt{x}} - 1 + \frac{1}{2}(x - 1)| < (x - 1)^2$ for any $|x - 1| \leq \frac{1}{2}$. Within the range of parameters that we specify, the residual terms (\ref{eq: index_bound_1_near_T_1_RHS_wa_res}) and (\ref{eq: index_bound_1_near_T_1_RHS_wb_res}) can be bounded by 
\begin{align}
    (\ref{eq: index_bound_1_near_T_1_RHS_wa_res}) \textrm{\ and\ } (\ref{eq: index_bound_1_near_T_1_RHS_wb_res}) \geq -\frac{1}{2}\nbt\left(\frac{1}{2}(\nat)^{-\frac{3}{2}} + \frac{1}{2}(\nbt)^{-\frac{3}{2}}\right) \epsilon. \label{eq: index_bound_1_near_T_1_RHS_res_bound}
\end{align} 
Indeed, 
\begin{align*}
    (\ref{eq: index_bound_1_near_T_1_RHS_wa_res}) & = f(T)\left(\xi_1(\nat)^{-\frac{5}{2}} - \xi_2(\nbt)^{-\frac{5}{2}}\right) \left(\wa - \nbt\right)^2\\
    & \geq - 2 f(T)(\nbt)^{-\frac{5}{2}}\left(\wa - \nbt\right)^2, \qquad\qquad\qquad\qquad (\textrm{since $|\xi_i| \leq 1$ and $\nbt \leq \nat$}) \\
    & \geq - 2(\nbt)^{-\frac{1}{2}}\frac{1}{f(T)}\left(\left|\frac{- \bar\mu_1(\nad) + \bar\mu_2(\nbd) + \Delta^T}{\left(\frac{1}{2}(\nat)^{-\frac{3}{2}} + \frac{1}{2}(\nbt)^{-\frac{3}{2}} \right)\nbt}\right| + \epsilon\right)^2,\\
    & \geq - 2(\nbt)^{-\frac{1}{2}}\frac{1}{f(T)}\left(\frac{\sqrt{f(T)\epsilon}}{4\sqrt{2}} + \frac{\sqrt{f(T)\epsilon}}{4\sqrt{2}}\right)^2,  \qquad\qquad (\textrm{by bounds on $\epsilon$ and (\ref{eq: bounds on Z})}) \\
    & = -(\nbt)^{-\frac{1}{2}}\frac{\epsilon}{4} = -\frac{1}{2}\nbt \times \frac{1}{2} (\nbt)^{-\frac{3}{2}}\epsilon,\\
    &\geq -\frac{1}{2}\nbt\left(\frac{1}{2}(\nat)^{-\frac{3}{2}} + \frac{1}{2}(\nbt)^{-\frac{3}{2}}\right) \epsilon. 
\end{align*}
Bounding (\ref{eq: index_bound_1_near_T_1_RHS_wb_res}) is similar. Now plugging (\ref{eq: index_bound_1_near_T_1_RHS_res_bound}) back to (\ref{eq: index_bound_1_near_T_1_RHS_wa_main}) and (\ref{eq: index_bound_1_near_T_1_RHS_wb_main}), we get rid of the residual terms and obtain a further relaxation of eq. (\ref{eq: index_bound_1_near_T_1}). 
\begin{align}
    (\ref{eq: index_bound_1_near_T_1}) &\subseteq \left\{\exists t\in  [\tau^2_T , T-1],\  \bamu_2(\wa)  -\bamu_2(\nbd) + \bamu_1(t - \wa) - \bamu_1(\nad) > \frac{1}{4}(\nbt)^{-\frac{1}{2}}\epsilon \right\},\label{eq: index_bound_1_near_T_1_1}\\  &\qquad\qquad\qquad\qquad\qquad\qquad\qquad\qquad\qquad\qquad\qquad\qquad (\textrm{when $i = 2$ and $a = \wa$})\NNN\\
  (\ref{eq: index_bound_1_near_T_1}) &\subseteq \left\{\exists t\in  [\tau^1_T , T-1],\  \bamu_1(T - \wb)  -\bamu_1(\nad) - \bamu_2(t - T + \wb) + \bamu_2(\nbd) > \frac{1}{4}(\nbt)^{-\frac{1}{2}}\epsilon \right\},\label{eq: index_bound_1_near_T_1_2}\\  &\qquad\qquad\qquad\qquad\qquad\qquad\qquad\qquad\qquad\qquad\qquad\qquad  (\textrm{when $i = 1$ and $a = T - \wb$}). \NNN
\end{align}
Our key observation is that $\bamu_i(t)$ comes with certain ``high-probability contraction'' property that makes the above events occur with vanishing probability. In the sequel, we apply a union bound to (\ref{eq: index_bound_1_near_T_1_1}) and (\ref{eq: index_bound_1_near_T_1_2}), respectively, and further reduce the task to showing each of the following events occurs with vanishing probability:
\begin{align}
    &\left\{\left|\bamu_2(\wa)  -\bamu_2(\nbd) \right| > \frac{1}{8}(\nbt)^{-\frac{1}{2}}\epsilon\right\},\label{eq: index_bound_1_near_T_concen_1}\\
    & \left\{\exists t\in  [\tau^2_T , T-1],\ \left|\bamu_1(t - \wa) - \bamu_1(\nad) \right| > \frac{1}{8}(\nbt)^{-\frac{1}{2}}\epsilon\right\},\label{eq: index_bound_1_near_T_concen_2}\\
    &\left\{\left|\bamu_1(T - \wb)  -\bamu_1(\nad)\right| > \frac{1}{8}(\nbt)^{-\frac{1}{2}}\epsilon\right\},\label{eq: index_bound_1_near_T_concen_3}\\
    & \left\{\exists t\in  [\tau^1_T , T-1],\ \left|- \bamu_2(t - T + \wb) + \bamu_2(\nbd) \right| > \frac{1}{8}(\nbt)^{-\frac{1}{2}}\epsilon\right\}\label{eq: index_bound_1_near_T_concen_4}.
\end{align}
To proceed, we develop a generic maximal inequality (cf. Lemma \ref{lem: maximal inequality}) and apply it to all the above events. Recall that we have confined ourselves to (\ref{eq: bounds on Z}), which in turn gives us 
\[
\left|\frac{\wa - \nbt}{\nbt}\right| \leq \frac{1}{\sqrt{f(T)\log f(T)}} + \frac{\epsilon}{f(T)}.
\]
Furthermore, 
\begin{align*}
    &\left|\frac{\tau^2_T - \wa - \nat}{\nat}\right| \leq \frac{1}{\sqrt{f(T)\log f(T)}} + \frac{\epsilon}{f(T)} + \eta_T, \qquad\qquad(\textrm{using $\nbt \leq \nat$})\\
    &\left|\frac{\tau^1_T - T + \wb - \nbt}{\nbt}\right| \leq \frac{1}{\sqrt{f(T)\log f(T)}} + \frac{\epsilon}{f(T)} + \eta_T, \\
    &\left|\frac{T - \wb - \nat}{\nat}\right| \leq \frac{1}{\sqrt{f(T)\log f(T)}} + \frac{\epsilon}{f(T)},  \qquad\qquad(\textrm{using $\nbt \leq \nat$})
\end{align*}
and 
\begin{align*}
\left|\frac{\nbd - \nbt}{\nbt}\right|, \left|\frac{\nad - \nat}{\nbt}\right| \leq \delta_T.
\end{align*}
As a result, taking 
\begin{align*}
    &s^i_1 = n^{\star}_{i, T}\left(1 - \frac{1}{\sqrt{f(T)\log f(T)}} - \frac{\epsilon}{f(T)} - \eta_T - \delta_T \right),\\
    & s^i_2 = n^{\star}_{i, T} \left(1 + \frac{1}{\sqrt{f(T)\log f(T)}} + \frac{\epsilon}{f(T)} + \eta_T + \delta_T \right),
\end{align*}
events (\ref{eq: index_bound_1_near_T_concen_1})-(\ref{eq: index_bound_1_near_T_concen_4})  can all be relaxed to 
\begin{align}\label{eq: index_bound_1_near_T_concen_all}
     \left\{\max_{s^i_1 \leq u < v \leq s^i_2} |\bamu_i(u) - \bamu_i(v)| > \frac{1}{8}(\nbt)^{-\frac{1}{2}}\epsilon \right\},
\end{align}
for $i = 1, 2$. Notice that $\frac{s^i_2 - s^i_1}{s^i_1} = o(1)$ by our definition on $\delta_T$ and $\eta_T$ and that $f(T) \to \infty$ as $T \to \infty$, and that $\epsilon$ is a constant. Applying Lemma \ref{lem: maximal inequality} to (\ref{eq: index_bound_1_near_T_concen_all}) and using the fact that $\nbt \leq \nat$ again, and we conclude that 
event (\ref{eq: index_bound_1_near_T_concen_all}) occurs with probability $O(\exp\left(O(\frac{s^i_1}{s^i_2 - s^i_1})\right))$, vanishing as $T \to \infty$.  Since (\ref{eq: bounds on Z}) occurs with high probability, each event (\ref{eq: index_bound_1_near_T_concen_1}) - (\ref{eq: index_bound_1_near_T_concen_4}) is contained in $(\ref{eq: bounds on Z})^c \cup (\ref{eq: index_bound_1_near_T_concen_all})$, hence vanishes as $T \to \infty$.  Combining the above, we conclude that (\ref{eq: index_bound_1_near_T_1}), and therefore \eqref{eq:index_bound_1_near_T} occurs with vanishing probability as $T \to \infty$.

\paragraph{Case 2. Treating (\ref{eq: index_bound_1_far})}
\begin{align}
     (\ref{eq: index_bound_1_far}) & = \left\{\exists t\leq \tau^i_T,\  \bamu_i(a) - \bamu_{-i}(t-a)  > \frac{f(t+1)}{\sqrt{t - a}} - \frac{f(t+1)}{\sqrt{a}}  \right\}\NNN\\
    & \ \ \subseteq \left\{\exists t\leq \tau^i_T,\  \bamu_i(a) - \bamu_{-i}(t-a)  > \frac{f(t)}{\sqrt{t - a}} - \frac{f(T)}{\sqrt{T - a}} + \frac{f(T)}{\sqrt{T - a}} - \frac{f(T)}{\sqrt{a}} \right\}.\label{eq: index_bound_1_far_T_1}
\end{align}
The treatment of $\frac{f(T)}{\sqrt{T - a}} - \frac{f(T)}{\sqrt{a}}$ is identical to that of (\ref{eq: index_bound_1_near_T_1_RHS_wa_res}) and (\ref{eq: index_bound_1_near_T_1_RHS_wb_res}), where we recall that assuming (\ref{eq: bounds on Z}), we have for $i = 2, a = \wa$ and $T$ sufficiently large
\begin{align}\label{eq: index_bound_1_far_near_a}
    \frac{f(T)}{\sqrt{T - \wa}} - \frac{f(T)}{\sqrt{\wa}} \geq \frac{1}{4}(\nbt)^{-\frac{1}{2}}\epsilon  -\bamu_1(\nad) + \bamu_2(\nbd),
\end{align}
and for $i = 1, a = T - \wb$ and $T$ sufficiently large, 
\begin{align}\label{eq: index_bound_1_far_near_b}
\frac{f(T)}{\sqrt{\wb}} - \frac{f(T)}{\sqrt{T - \wb}}  \geq \frac{1}{4}(\nbt)^{-\frac{1}{2}}\epsilon +\bamu_1(\nad) - \bamu_2(\nbd),
\end{align}
Plugging back into (\ref{eq: index_bound_1_far_T_1}), and we have when $ i = 2, a = \wa$, 
\begin{align}
    (\ref{eq: index_bound_1_far_T_1}) & \subseteq \left\{ \exists t\leq \tau^2_T,\  \bamu_2(\wa) - \bamu_2(\nbd) - \bamu_{1}(t-\wa) + \bamu_1(\nad)  > \frac{f(t)}{\sqrt{t - \wa}} - \frac{f(T)}{\sqrt{T - \wa}} + \frac{1}{4}(\nbt)^{-\frac{1}{2}}\epsilon \right\},\NNN\\
    &\subseteq \left\{  \bamu_2(\wa) - \bamu_2(\nbd) >  \frac{1}{4}(\nbt)^{-\frac{1}{2}}\epsilon \right\}\label{eq: index_bound_1_far_T_1_far_near_a_concen}\\
    & \qquad\qquad \bigcup \left\{ \exists t\leq \tau^2_T,\  - \bamu_{1}(t-\wa) + \bamu_1(\nad)  > \frac{f(t)}{\sqrt{t - \wa}} - \frac{f(T)}{\sqrt{T - \wa}} \right\}, \label{eq: index_bound_1_far_T_1_far_near_a_lil}
\end{align}
by union bound. Similarly for $i = 1, a = T - \wb$,
\begin{align}
    (\ref{eq: index_bound_1_far_T_1}) & \subseteq  \left\{  \bamu_1(T - \wb) - \bamu_1(\nad) >  \frac{1}{4}(\nbt)^{-\frac{1}{2}}\epsilon \right\}\label{eq: index_bound_1_far_T_1_far_near_b_concen}\\
    & \qquad\qquad \bigcup \left\{ \exists t\leq \tau^1_T,\  - \bamu_{2}(t - T + \wb) + \bamu_2(\nbd)  > \frac{f(t)}{\sqrt{t - T + \wb}} - \frac{f(T)}{\sqrt{\wb}} \right\}. \label{eq: index_bound_1_far_T_1_far_near_b_lil}
\end{align}
Note that  (i) (\ref{eq: index_bound_1_far_T_1_far_near_a_concen}) implies (\ref{eq: index_bound_1_near_T_concen_1}) and (ii) (\ref{eq: index_bound_1_far_T_1_far_near_b_concen}) implies (\ref{eq: index_bound_1_near_T_concen_3}), for which we have already shown to have vanishing probability as $T \to \infty$. Thus it boils down to treat events (\ref{eq: index_bound_1_far_T_1_far_near_a_lil}) and (\ref{eq: index_bound_1_far_T_1_far_near_b_lil}).

By Assumption \ref{assump: index}, for any $t < T$, $\frac{f(t)}{f(T)} \geq \left(\frac{t}{T}\right)^\beta$ for index $\beta < \frac{1}{2}$. Therefore,
\begin{align*}
    \frac{f(t)}{\sqrt{t - a}}\left(\frac{f(T)}{\sqrt{T - a}}\right)^{-1} & = \frac{f(t)}{f(T)} \cdot\frac{\sqrt{T - a}}{\sqrt{t - a}}\\
    & \geq  \left(\frac{t}{T}\right)^\beta\frac{\sqrt{T - a}}{\sqrt{t - a}} \geq   \left(\frac{t - a}{T - a}\right)^\beta\frac{\sqrt{T - a}}{\sqrt{t - a}} =  \left(\frac{T - a}{t - a}\right)^{-\beta + \frac{1}{2}}\\
    & \geq 1 + \left(\frac{1}{2} - \beta\right)\frac{T - t}{t - a},
\end{align*}
for any $ a< t$. In the case of (\ref{eq: index_bound_1_far_T_1_far_near_a_lil}), $i = 2$ and $a = \wa$, we know that $t \leq \tau^2_T = T - \eta_T \nat$. Meanwhile (under (\ref{eq: bounds on Z})) $\wa \geq \left(1 - \frac{1}{\sqrt{f(T)\log f(T)}} - \frac{\epsilon}{f(T)}\right)\nbt$, we thus have
\begin{align*}
\frac{T - t}{t - \wa} &\geq \frac{T - (T - \eta_T \nat)}{T - \eta_T \nat - \left(1 - \frac{1}{\sqrt{f(T)\log f(T)}} - \frac{\epsilon}{f(T)}\right)\nbt} = \frac{\eta_T \nat}{\nat + \left(\frac{1}{\sqrt{f(T)\log f(T)}} + \frac{\epsilon}{f(T)}\right)\nbt},\\
&\geq \frac{1}{2}\eta_T.  \qquad\qquad\qquad\qquad (\textrm{since $\nbt \leq \nat$ and for $T$ sufficiently large})
\end{align*}
Similarly, in the case of (\ref{eq: index_bound_1_far_T_1_far_near_b_lil}), i.e. $i = 1$ and $a = T - \wb$, we know that $t \leq \tau^1_T = T - \eta_T \nbt$, we also have for $T$ sufficiently large 
\[
\frac{T - t}{t - T + \wb} \geq \frac{1}{2}\eta_T.
\]
In both cases, we have for $T$ sufficiently large, for any $ t\leq \tau^i_T$, under (\ref{eq: bounds on Z}),
\[
\frac{f(t)}{\sqrt{t - a}} - \frac{f(T)}{\sqrt{ T - a}} \geq \frac{1 - 2\beta}{10}\eta_T \frac{f(t)}{\sqrt{t - a}}.
\]
Plugging back into (\ref{eq: index_bound_1_far_T_1_far_near_a_lil}), we have that 
\begin{align}
    (\ref{eq: index_bound_1_far_T_1_far_near_a_lil}) &\subseteq \left\{ \exists t\leq \tau^2_T,\  - \bamu_{1}(t-\wa) + \bamu_1(\nad)  > \frac{1 - 2\beta}{10}\eta_T \frac{f(t)}{\sqrt{t - \wa}} \right\},\NNN\\
    &\subseteq \left\{ \exists t\leq \tau^2_T,\  \left|-\mu_1 + \bamu_1(\nad) \right|  > \frac{1 - 2\beta}{20}\eta_T \frac{f(t)}{\sqrt{t - \wa}} \right\}\NNN\\
    & \qquad\qquad \bigcup \left\{ \exists t\leq \tau^2_T,\  \left| - \bamu_{1}(t-\wa)  + \mu_1 \right|  > \frac{1 - 2\beta}{20}\eta_T \frac{f(t)}{\sqrt{t - \wa}} \right\}, \NNN\\
    &\subseteq \left\{   \left|-\mu_1 + \bamu_1(\nad) \right|  > \frac{1 - 2\beta}{20}\eta_T \frac{f(\tau^2_T)}{\sqrt{\tau^2_T - \wa}} \right\} \label{eq: a_lil_near}\\
    & \qquad\qquad \bigcup \left\{ \exists t\leq \tau^2_T,\  \left| - \bamu_{1}(t-\wa)  + \mu_1 \right|  > \frac{1 - 2\beta}{20}\eta_T \frac{f(t)}{\sqrt{t - \wa}} \right\}. \label{eq: a_lil_far}
\end{align}
By Assumption \ref{assump: index} and the fact that $T > \tau^2_T$, the RHS of (\ref{eq: a_lil_near}) is lower bounded by $\frac{1 - 2\beta}{20}\eta_T \frac{f(T)}{\sqrt{T}}$. By the sub-Gaussianity of $\bar\mu_1 - \mu_1$ (Lemma~\ref{lem: sub_gaussian}), 
\[
\PP\left(\left|\left| -\mu_1 + \bamu_1(\nad) \right|  > \frac{1 - 2\beta}{20}\eta_T \frac{f(T)}{\sqrt{T}} \right|\right) \leq 2\exp\left(- \frac{(1 - 2 \beta)^2 \eta_T^2 f(T)^2\nad}{2 \times 20^2 (\sigma_2)^2 T }\right),
\]
which vanishes as $T \to \infty$ because $\frac{T}{\nad} \sim \frac{T}{\nat} = \frac{\nat + \nbt}{\nat} \leq 2$ for $T$ sufficiently large, and $\lim_{T \to \infty}f(T)\eta_T = \infty$ by definition of $\eta_T$. Next let's turn to (\ref{eq: a_lil_far}). 
\begin{align}
   (\ref{eq: a_lil_far}) &\subseteq \left\{ \exists t\geq \wa + 1,\  \left| - \bamu_{1}(t-\wa)  + \mu_1 \right|  > \frac{1 - 2\beta}{20}\eta_T \frac{f(t)}{\sqrt{t - \wa}} \right\},\NNN\\
   &\subseteq \left\{ \exists \wa +1 \leq t\leq \wa + q_T,\  \left| - \bamu_{1}(t-\wa)  + \mu_1 \right|  > \frac{1 - 2\beta}{20}\eta_T \frac{f(t)}{\sqrt{t - \wa}} \right\}, \label{eq: a_lil_far_1}\\
   &\ \ \ \ \bigcup \left\{ \exists t\geq \wa + q_T,\  \left| - \bamu_{1}(t-\wa)  + \mu_1 \right|  > \frac{1 - 2\beta}{20}\eta_T \frac{f(t)}{\sqrt{t - \wa}} \right\}. \label{eq: a_lil_far_2}
\end{align}
Here $q_T = \log(2 + \log(2 + \log(2 +f(\sqrt{\nbt}))))$. Recall by definition, $\eta_T = \frac{100 \sigma}{1 - 2\beta}\frac{\sqrt{\log(2 + \log(2 +f(\sqrt{\nbt})))}}{f(\sqrt{\nbt})}$. Under (\ref{eq: bounds on Z}), we have $\wa > \sqrt{\nbt}$ for sufficiently large $T$. Thus for all $t$ of consideration in (\ref{eq: a_lil_far_1}) and (\ref{eq: a_lil_far_2}), it all holds true that 
\[
\eta_T \geq \frac{100 \sigma}{1 - 2\beta}\frac{\sqrt{\log(2 + \log(2 +f(t)))}}{f(t)}, 
\]
due to the monotonicity (decreasing) of $\frac{\sqrt{\log(2 + \log(2 + f(t)) )}}{f(t)}$ from Assumption \ref{assump: index}. By a union bound
\begin{align}
     \PP((\ref{eq: a_lil_far_1})) &\leq \PP\left(\exists \wa +1 \leq t\leq \wa + q_T,\  \left| - \bamu_{1}(t-\wa)  + \mu_1 \right|  > 5 \frac{\sqrt{\log(2 + \log(2 +f(t)))}}{\sqrt{t - \wa}} \right),\NNN\\
    & \leq \sum_{j = 1}^{q_T}\PP\left(\left|-\bamu_1(j) + \mu_1\right| > \frac{5 \sigma \sqrt{\log(2 + \log(2 +f( j + \wa)))}}{\sqrt{j}}\right),\NNN\\
    & \leq \sum_{j = 1}^{q_T}\PP\left(\left|-\bamu_1(j) + \mu_1\right| > \frac{5 \sigma \sqrt{\log(2 + \log(2 +f( \sqrt{\nbt})))}}{\sqrt{j}}\right), \NNN\\
    &\qquad\qquad\qquad\qquad\qquad\qquad\qquad\qquad\qquad (\textrm{under (\ref{eq: bounds on Z}) and $T$ sufficiently large})\NNN\\
    & \leq 2\sum_{j = 1}^{q_T}\exp\left(-\frac{j}{2\sigma^2}\frac{25 \sigma^2 \left(\log(2 + \log(2 +f( \sqrt{\nbt})))\right)}{j}\right),  \qquad\qquad\qquad\qquad (\textrm{sub-Gaussian})\NNN\\
    &= 2 q_T \exp\left(- \frac{25}{2}\sigma^2 \left(\log(2 + \log(2 +f( \sqrt{\nbt})))\right)\right)\NNN\\
    &= 2\log(2 + \log(2 + \log(2 +f(\sqrt{\nbt})))) \exp\left(- \frac{25}{2}\sigma^2 \left(\log(2 + \log(2 +f( \sqrt{\nbt})))\right)\right),\NNN
\end{align}
which occurs with vanishing probability as $\nbt \to \infty$ with $T$. While for (\ref{eq: a_lil_far_2}), by bounds on $\eta_T$, we have
\begin{align}
    (\ref{eq: a_lil_far_2}) &= \left\{ \exists j \geq q_T,\  \left| - \bamu_{1}(j)  + \mu_1 \right|  > \frac{1 - 2\beta}{20}\eta_T \frac{f(j + \wa)}{\sqrt{j}} \right\}, \NNN\\
    &\subseteq \left\{ \exists j \geq q_T,\  \left| - \bamu_{1}(j)  + \mu_1 \right|  > 5 \sigma \frac{\sqrt{\log(2 + \log (2 + j))}}{\sqrt{j}} \right\}. \NNN
\end{align}
Set $\varrho_T \triangleq \frac{1}{\log(2 + 2 q_T)}$. Since $q_T = \omega(1)$, $\varrho_T = o(1)$. We have, for $j \geq q_T$, $\log(2 + 2 j) \geq (\varrho_T)^{-1}$. And thus, 
\[
1.5\sigma^T_1 \sqrt{2.5  \log \left(\frac{\log(2 + 1.25 j)}{\varrho_T}\right)} \leq 1.5\sigma^T_1 \sqrt{ 2 \times 2.5  \log \left(\log(2 + 2 j)\right)} \leq 5 \sigma\sqrt{\log(2 + \log(2 + j))} .
\]
By Lemma \ref{lem: anytime-LIL} with $\theta = 0.25$ and $\delta = \varrho_T$, the above implies that $\PP(\eqref{eq: a_lil_far_2}) \leq \frac{2.25}{0.25}\left(\frac{1}{\varrho_T \log (1.25)}\right)^{1.25}$, which vanishes as $T \to \infty$. Now combining the above, we have that (\ref{eq: a_lil_far}) occurs with vanishing probability as $T \to \infty$. Combining with results on (\ref{eq: a_lil_near}), together they imply that (\ref{eq: index_bound_1_far_T_1_far_near_a_lil}) occurs with vanishing probability as $T \to \infty.$ 

Argument for (\ref{eq: index_bound_1_far_T_1_far_near_b_lil}) is nearly identical, where we decompose according to 
\begin{align}
   &(\ref{eq: index_bound_1_far_T_1_far_near_b_lil}) \NNN\\
  & \qquad  \subseteq \left\{   \left|-\mu_2 + \bamu_2(\nbd) \right|  > \frac{1 - 2\beta}{20}\eta_T \frac{f(\tau^1_T)}{\sqrt{\tau^1_T - T + \wb}} \right\} \label{eq: b_lil_near}\\
    & \qquad\qquad \bigcup \left\{ \exists t\leq \tau^1_T,\  \left| - \bamu_{2}(t- T + \wb)  + \mu_2 \right|  > \frac{1 - 2\beta}{20}\eta_T \frac{f(t)}{\sqrt{t - T +  \wb}} \right\}. \label{eq: b_lil_far}
\end{align}
Recall that $\tau^1_T = T - \eta_T \nbt$. We relax RHS in (\ref{eq: b_lil_near}) to $\frac{1 - 2\beta}{20}\eta_T\frac{f(T)}{\sqrt{\wb}}$, noting that $|\frac{\wb - \nbt}{\nbt}| \to 1$ and $|\frac{\nbd - \nbt}{\nbt}| \to 1$ as $T  \to \infty$ under (\ref{eq: bounds on Z}), then applying Chebyshev's inequality to get
\begin{align*}
    \PP\left((\ref{eq: b_lil_near})\right) \leq \frac{20^2 \sigma^2}{(1 - 2 \beta)^2}\cdot \frac{1}{(\eta_T f(T))^2},
\end{align*}
which vanishes as $T \to \infty$ since $\eta_T f(T) \to \infty$. We further decompose (\ref{eq: b_lil_far}) as
\begin{align}
   & (\ref{eq: b_lil_far}) \NNN\\
   &\subseteq \left\{ \exists T - \wb +1 \leq t\leq T - \wb + q_T,\  \left| - \bamu_{2}(t-  T + \wb)  + \mu_2 \right|  > \frac{1 - 2\beta}{20}\eta_T \frac{f(t)}{\sqrt{t - T + \wb}} \right\}, \label{eq: b_lil_far_1}\\
   &\ \ \ \ \bigcup \left\{ \exists t\geq T - \wb + q_T,\  \left| - \bamu_{2}(t- T + \wb)  + \mu_2 \right|  > \frac{1 - 2\beta}{20}\eta_T \frac{f(t)}{\sqrt{t - T +  \wb}} \right\}. \label{eq: b_lil_far_2}
\end{align}
Recall that $T = \nat + \nbt$ and $\left|\frac{\wb - \nbt}{\nbt}\right| \to 1$ as $T \to \infty$, assuming (\ref{eq: bounds on Z}). Thus $T - \wb > \sqrt{\nbt}$ for $T$ sufficiently large. Therefore, for all $t$ in (\ref{eq: b_lil_far_1}) and (\ref{eq: b_lil_far_2}), we have 
\[
\eta_T \geq \frac{100 \sigma}{1 - 2 \beta}\frac{\sqrt{\log(2 + \log(2 + f(t)))}}{f(t)}.
\]
That (\ref{eq: b_lil_far_1}) occurs with vanishing probability again follows from a union bound combined with the sub-Gaussianity of $\bar\mu_2(j) - \mu_2$, and 
\begin{align*}
    (\ref{eq: b_lil_far_2}) \subseteq \left\{\exists j \geq q_T, |-\bamu_2(j) + \mu_2| > 5 \sigma \frac{\sqrt{\log(2 + \log(2 + f(j)))}}{\sqrt{j}}\right\}
\end{align*}
occurring with vanishing probability due to Lemma \ref{lem: anytime-LIL}. These, altogether, conclude the treatment of (\ref{eq: index_bound_1_far}). 

Combining Case 1 and Case 2, we complete the proof of Lemma \ref{lem: high-prob-N1-n1}.
\end{proof}

\subsection{Proof of other lemmas}\label{app: other proof}

\noindent\textbf{Proof of Lemma~\ref{lem: maximal inequality}.} The statement of Lemma~\ref{lem: maximal inequality} is as follows.
\begin{quote}
   Suppose $Y_i, i \geq 1$ are i.i.d. centered $\sigma$-sub-Gaussian. Then for any $1 \leq s_1 < s_2$, we have
\[
\PP\left(\max_{s_1 \leq u < v \leq s_2} \left| \frac{\sum_{i = 1}^u Y_i}{u} - \frac{\sum_{i = 1}^v Y_i}{v}\right| >  a \right) \leq 8\exp\left(-\frac{a^2 s_1^2}{72 \sigma^2 (s_2 - s_1)} \right).
\]    
\end{quote}

\begin{proof}
    Note that
    \begin{align*}
        \max_{s_1 \leq u < v \leq s_2} \left| \frac{\sum_{i = 1}^u Y_i}{u} - \frac{\sum_{i = 1}^v Y_i}{v}\right| &\leq \max_{s_1 \leq u \leq s_2} \left| \frac{\sum_{i = 1}^u Y_i}{u} - \frac{\sum_{i = 1}^{s_1} Y_i}{s_1}\right| + \max_{s_1 \leq v \leq s_2} \left| \frac{\sum_{i = 1}^{s_1} Y_i}{s_1} - \frac{\sum_{i = 1}^v Y_i}{v}\right| \\
        & = 2\max_{s_1 \leq u \leq s_2} \left| \frac{\sum_{i = 1}^u Y_i}{u} - \frac{\sum_{i = 1}^{s_1} Y_i}{s_1}\right|,\\
        & = 2\max_{s_1 \leq u \leq s_2} \left| \frac{\sum_{i = 1}^{s_1} Y_i + \sum_{i = s_1+1}^u Y_i}{s_1 + u - s_1} - \frac{\sum_{i = 1}^{s_1} Y_i}{s_1}\right|,\\
        & \leq \frac{2(s_2 - s_1)\left|\sum_{i = 1}^{s_1} Y_i\right|}{s_1 s_2} + \frac{2}{s_1}\max_{s_1 \leq u < s_2} \left|\sum_{i = s_1+1}^u Y_i\right|.
    \end{align*}
By Lemma~\ref{lem: sub_gaussian}, $\sum_{i = 1}^{s_1} Y_i$ is also sub-Gaussian with variance proxy $\sigma^2 s_1$, which implies that
\begin{align*}
        \PP\left(\frac{2(s_2 - s_1)\left|\sum_{i = 1}^{s_1} Y_i\right|}{s_1 s_2} \geq \frac{a}{2} \right) \leq  2 \exp\left(- \frac{1}{2\sigma^2 s_1}\cdot\frac{a^2 s_1^2 s_2^2}{16(s_2 -s_1)^2}\right) = 2 \exp\left(-\frac{a^2 s_1 s_2^2}{32\sigma^2(s_2 - s_1)^2}\right).
    \end{align*}
    On the other hand, apply the Etemadi's inequality (Lemma~\ref{lm: etemadi}):
    \begin{align*}
        \PP\left(\max_{s_1 \leq u < s_2} \left|\sum_{i = s_1+1}^u Y_i\right| > \frac{s_1 a }{2} \right) &=  \PP\left(\max_{1 \leq u \leq s_2-s_1} \left|\sum_{i = 1}^u Y'_i\right| > \frac{a s_1  }{2} \right),\\
        & \leq 3 \max_{1 \leq u \leq s_2-s_1} \PP\left(\left|\sum_{i = 1}^u Y'_i\right|  >\frac{a s_1 }{6}\right),\\
        & \leq 6 \max_{1 \leq u \leq s_2-s_1} \exp\left(-\frac{a^2 s_1^2}{72 \sigma^2 u} \right), \\
        & =  6\exp\left(-\frac{a^2 s_1^2}{72 \sigma^2 (s_2 - s_1)} \right),
    \end{align*}
    where $Y'_i$ are \emph{i.i.d.} copies of $Y_i$, the second inequality is the Etemadi's inequality, and the third inequality follows from sub-Gaussianity of $\sum_{i = 1}^u Y'_i$.  Combining the two terms and we get that 
        \begin{align*}
        & \PP\left(\max_{s_1 \leq u < v \leq s_2} \left| \frac{\sum_{i = 1}^u Y_i}{u} - \frac{\sum_{i = 1}^v Y_i}{v}\right| \geq a \right) \\
        & \ \ \leq \PP\left(\frac{2(s_2 - s_1)\left|\sum_{i = 1}^{s_1} Y_i\right|}{s_1 s_2} \geq \frac{a}{2} \right)  + \PP\left(\max_{s_1 \leq u < s_2} \left|\sum_{i = s_1+1}^u Y_i\right| > \frac{s_1 a }{2} \right),\\
        & \leq 8\exp\left(-\frac{a^2 s_1^2}{72 \sigma^2 (s_2 - s_1)} \right).
    \end{align*} 
\end{proof}

\noindent\textbf{Proof of Lemma~\ref{lem: bar-mu-clt}.} The statement of Lemma~\ref{lem: bar-mu-clt} is as follows.
\begin{quote}
   Let $M^\delta_T \triangleq \frac{- \bar\mu_1(\nad) + \bar\mu_2(\nbd) + \Delta^T_2}{\left(\frac{1}{2}(\nat)^{-\frac{3}{2}} + \frac{1}{2}(\nbt)^{-\frac{3}{2}} \right)\nbt}$. Then $M^{\delta}_T \xlongrightarrow[]{d} \mathcal{N}\left(0, \frac{4\lx\sigma_1^2 + 4 \sigma_2^2}{\left(1 + (\lx)^{\frac{3}{2}}\right)^2}\right).$
\end{quote}

\begin{proof}
    Note that $M^\delta_T$ has mean zero, and variance
    \[
    \sigma^2_{\delta, T} = \frac{(\sigma^T_1)^2\frac{1}{\nad} + (\sigma^T_2)^2\frac{1}{\nbd}}{\left(\frac{1}{2}(\nat)^{-\frac{3}{2}} + \frac{1}{2}(\nbt)^{-\frac{3}{2}} \right)^2(\nbt)^2}.
    \]
    Since the arm rewards are sub-Gaussian (Assumption \ref{assump: distribution}), the Lyapunov condition of the triangular array CLT is satisfied, and by Lemma~\ref{lem: triangular CLT} we have $\frac{1}{\sigma_{\delta, T}} M^{\delta}_T \xlongrightarrow[]{d} \mathcal{N}(0, 1)$. Notice that 
    \[
    \lim_{T \to \infty} \sigma_{\delta, T} = \frac{2\sqrt{\sigma_1^2 \lx + \sigma_2^2 }}{1 + (\lx)^{\frac{3}{2}}}.
    \]
    The desired result follows from Slutsky's theorem (Lemma~\ref{lem:slutsky}).
\end{proof}

\begin{lemma}\label{lem: wat-close-to-nat}
Let $M^\delta_T$ as defined in the previous lemma. It holds true that
\begin{align*}
     \PP\left(\left|M^\delta_T\right| \geq m\right) \leq  2\exp\left(-\frac{m^2}{32 \sigma^2}\right).
\end{align*}
\end{lemma}
\begin{proof}
   The random variable $M^\delta_T$ has zero mean. It's variance is 
   \begin{align*}
       \textrm{Var}\left(\frac{\bar\mu_1(\nad) - \bar\mu_2(\nbd) - \Delta^T}{\left(\frac{1}{2}(\nat)^{-\frac{3}{2}} + \frac{1}{2}(\nbt)^{-\frac{3}{2}} \right)\nbt}\right) &= \frac{\frac{(\sigma^T_{1})^2}{\nad} + \frac{(\sigma^T_{2})^2}{\nbd}}{\left(\frac{1}{2}(\nat)^{-\frac{3}{2}} + \frac{1}{2}(\nbt)^{-\frac{3}{2}} \right)^2(\nbt)^2} \\
       & \leq {16 \sigma^2},
   \end{align*}
   where we apply Lemma \ref{lem: nstar_diverge} with $\nat > \nbt$ and we use the fact that $\nbd \geq \frac{1}{2}\nbt$ since $\delta_T \geq \frac{1}{2}$ by definition, and that $0 < \sigma_1 , \sigma_2 \leq \sigma$ by Assumption~\ref{assump: distribution}. By Lemma~\ref{lem: sub_gaussian}, $M^\delta_T$ is sub-Gaussian. Combining the above thus leads to the desired probability bound. 
\end{proof}

\section{Supplementary Materials on Sampling Bias}\label{app:sampling_bias}
\subsection{Sketch analysis of the stylized model}
In this section we provide a sketch analysis of the sample mean in the stylized model. Recall that in the case of {\sf canonical UCB}, we have $f(t) = \sqrt{\rho \log T}$ and the stylized model specifies  a sequence $\delta_T: \delta_T = \omega\left((\log T)^{-\frac{1}{2}}\right)$ and $\delta_T = o(1)$ as prescribed input, and do: 
\begin{enumerate}
    \item Generate $n^{\delta}_{i, T} \triangleq (1 - \delta_T)n^{\star}_{i, T}$\ \ \emph{i.i.d.} rewards from arm $i$, $i = 1, 2$
    \item Compute the normalized sample mean from the two arms:
    \begin{align*}
        Z^{\delta}_{i, T} \triangleq \sqrt{n^{\delta}_{i, T}}\left(\bamu^T_i\left(n^{\delta}_{i, T}\right) - \mu^T_{i}  \right), \ i = 1, 2.
    \end{align*}
    \item Compute 
    \begin{align*}
        \tilde N_{2,T}&= \nbt\left(1 +\frac{2\left(Z^{\delta}_{2, T} - Z^{\delta}_{1, T}\sqrt{\lx}\right)}{\left(1 + (\lx)^{\frac{3}{2}}\right)\sqrt{\rho\log T}}\right), \ \ \tilde N_{1, T} = T - \tilde N_{2, T}.
    \end{align*}
    \item Sample $\tilde N_{i, T} - n^{\delta}_{i, T}$ more \emph{i.i.d.} rewards from the two arms, respectively.
\end{enumerate}
The sample mean in this stylized model, denoted by $\tilde \mu_{i, T}$, is the combination of 
\begin{enumerate}
    \item $n^{\delta}_{i, T}$ number of data collected in Step 1, with sample mean $\hat\mu = \bamu^T_i\left(n^{\delta}_{i, T}\right)$
    \item $\left(\tilde N_{i, T} - n^{\delta}_{i, T}\right)$  number of data collected in Step 4, with sample mean $\hat\mu'$.
\end{enumerate}
Thus we have 
\begin{align}
    \tilde \mu_{2, T} &= \frac{n^{\delta}_{2, T} \hat\mu + \left(\tilde N_{2, T} - n^{\delta}_{2, T}\right)\hat\mu'}{\tilde N_{2, T}},\NNN\\
    &= \mu_2 + \frac{\sqrt{(1 - \delta_T)\nbt} Z^{\delta}_{2, T} + \sqrt{\left|\delta_T + \frac{2\left(Z^{\delta}_{2, T} - Z^{\delta}_{1, T}\sqrt{\lx}\right)}{\left(1 + (\lx)^{\frac{3}{2}}\right)\sqrt{\rho\log T}}\right|\nbt} Z'_2}{\nbt\left(1 + \frac{2\left(Z^{\delta}_{2, T} - Z^{\delta}_{1, T}\sqrt{\lx}\right)}{\left(1 + (\lx)^{\frac{3}{2}}\right)\sqrt{\rho\log T}}\right)},\NNN\\
    &= \mu_2 + \frac{\sqrt{(1 - \delta_T)\nbt} Z^{\delta}_{2, T}}{\nbt} - \frac{\sqrt{(1 - \delta_T)\nbt}Z^{\delta}_{2, T}}{\nbt}\frac{\frac{2\left(Z^{\delta}_{2, T} - Z^{\delta}_{1, T}\sqrt{\lx}\right)}{\left(1 + (\lx)^{\frac{3}{2}}\right)\sqrt{\rho\log T}}}{1 + \frac{2\left(Z^{\delta}_{2, T} - Z^{\delta}_{1, T}\sqrt{\lx}\right)}{\left(1 + (\lx)^{\frac{3}{2}}\right)\sqrt{\rho\log T}}} \label{eq: key term-bias}\\   &\qquad\qquad\qquad\qquad + \frac{\sqrt{\left|\delta_T +\frac{2\left(Z^{\delta}_{2, T} - Z^{\delta}_{1, T}\sqrt{\lx}\right)}{\left(1 + (\lx)^{\frac{3}{2}}\right)\sqrt{\rho\log T}}\right|\nbt} Z'_2}{\nbt\left(1 + \frac{2\left(Z^{\delta}_{2, T} - Z^{\delta}_{1, T}\sqrt{\lx}\right)}{\left(1 + (\lx)^{\frac{3}{2}}\right)\sqrt{\rho\log T}}\right)}, \label{eq: no-bias-new-sample}
\end{align}
where $Z^{\delta}_{i, T} \xlongrightarrow[]{d} \mathcal{N}\left(0, \sigma_i^2\right)$ and $Z'_2 \triangleq \sqrt{\tilde N_{i, T} - n^{\delta}_{i, T}}\left(\hat\mu' - \mu_2\right) \xlongrightarrow[]{d} 
 \mathcal{N}\left(0, \sigma_2^2\right)$ by CLT. The three random variables are independent. Our key observation is that the last term in \eqref{eq: key term-bias} is the only term that contributes bias, because in the term of \eqref{eq: no-bias-new-sample}, $Z'_2$ is asymptotically normal and independent of both $Z^{\delta}_{i, T}$, while the other term in \eqref{eq: key term-bias} is also asymptotically normal and unbiased. With some straightforward calculation, we derive (up to the leading order) the precise random variable of interested in \eqref{eq: key term-bias} can be explicitly written as 
 \begin{align}
     \frac{2}{\left(1 + (\lx)^{\frac{3}{2}}\right)\sqrt{\rho}}\left(Z^{\delta}_{2, T} - Z^{\delta}_{1, T}\sqrt{\lx}\right)Z^{\delta}_{2, T} \frac{1}{\sqrt{\nbt\log T}}.
 \end{align}
In particular, we use again that $Z^{\delta}_{1, T}, Z^{\delta}_{2, T}$ are independent and that ${\rm Var}\left(Z^{\delta}_{2, T}\right)^2 = \sigma_2^2$, we conclude that the leading bias term is
\[
- \frac{2\sigma_2^2}{\left(1 + (\lx)^{\frac{3}{2}}\right)\sqrt{\rho \nbt \log T}}.
\]
Similarly, we can derive the leading bias term for arm $1$, which is
\[
- \frac{2\sigma_1^2}{\left(1 + (\lx)^{-\frac{3}{2}}\right) \sqrt{\rho \nat \log T}}.
\]
Combining with Lemma \ref{lem: three gap regimes}, we effectively derive the bias in the stylized model, and hence in Conjecture \ref{conjecture:sampling-bias-UCB}. In general, a rigorous characterization of the sample bias in the true UCB bandit system is challenging and beyond the scope of this paper, we hence leave it for further study.\\

\noindent{\bf Numerics.} We conduct numerical experiments of two-armed stochastic bandits under the {\sf UCB1} algorithm ($f(t)=\sqrt{2\log t}$), and consider $\mathcal P_i$ being $\mathcal N(\mu_i,\sigma_i^2)$, $i=1,2$. There are 10000 repetitions for each $T$ in a range of values from $10^3$ to $10^{13}$. To improve the efficiency of the simulation for large values of $T$, we leverage the typical deviation characterization from Theorem \ref{thm:N1} to pull arms in a carefully chosen batch size that is just smaller (in scaling) than the typical deviation of that arm's number of pulls. Thus we only need to generate one total reward (a Normal random variable) for the batched pull, hence effectively speeding up the simulation. In the moderate-small arm gap regime, whenever the algorithm chooses an arm, it is pulled in a batch size of $\frac{0.02T}{\log T}$. In the large gap regime, we use a batch size of $\frac{0.02 T}{\log T}$ only when the superior arm is pulled.

For each experiment, we calculate the sample means $\bar\mu_{1,T}$ and $\bar\mu_{2,T}$. Then we calculate the average value of the sample means under 10000 repetitions for each $T$. Denote the average value of the sample means under $T$ by $\hat\mu_{1,T}$ and $\hat\mu_{2,T}$. We then calculate the empirical biases $(\hat\mu_{1,T}-\mu_1)$ and $(\hat\mu_{2,T}-\mu_2)$. Next, we present the values of the sample biases after some proper scaling from the characterization in Conjecture \ref{conjecture:sampling-bias-UCB}, and compare them with the constant factors in Conjecture \ref{conjecture:sampling-bias-UCB}.

First consider the small gap regime. We choose $\mu_1=\mu_2=1$ and $\sigma_1=\sigma_2=\sigma$ for $\sigma=0.5,0.7,0.9$. According to Conjecture \ref{conjecture:sampling-bias-UCB}, the proper scaling of the empirical bias should be $\sqrt{T\log T}$, and the constant should be $-\sigma^2=-0.25, -0.49, -0.81$, respectively, for $\sigma=0.5, 0.7, 0.9$. The comparison between the scaled empirical bias under 10000 repetitions and Conjecture \ref{conjecture:sampling-bias-UCB} is presented in Figure~\ref{fig:scaled-bias-small-gap}. As the figure illustrates, the scaled empirical bias is close to the conjectured value in Conjecture \ref{conjecture:sampling-bias-UCB}.

\begin{figure}
	\centering
	\includegraphics[width=0.8\textwidth]{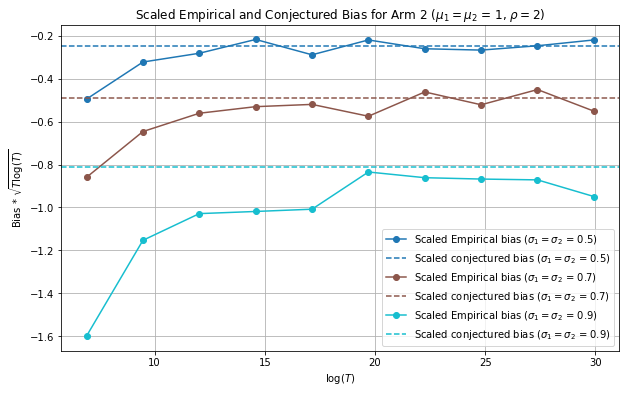}
	\caption{Scaled empirical bias of arm 2 under 10000 repetitions, $(\hat\mu_2(T)-\mu_2)\sqrt{T\log T}$, versus scaled (by $\sqrt{T\log T}$) conjectured bias of arm 2 in Conjecture \ref{conjecture:sampling-bias-UCB}, $\sigma_2^2$, for different horizon length $T$ We fix $\mu_1=\mu_2=1$ and $\rho=2$, and vary the values of $\sigma_1=\sigma_2$, represented by each curve.}\label{fig:scaled-bias-small-gap}
\end{figure}

We also test the large gap regime. We fix $\mu_1=2$, $\sigma_1=\sigma_2=1$, and choose $\mu_2=0, 1, 1.5$. According to Conjecture \ref{conjecture:sampling-bias-UCB}, the proper scaling of the empirical bias of arm 1 should be $\log T$, and the constant should be $-\sigma_2^2(\mu_1-\mu_2)$. The comparison between the scaled empirical bias under 10000 repetitions and Conjecture \ref{conjecture:sampling-bias-UCB} is presented in Figure~\ref{fig:scaled-bias-large-gap}.
\begin{figure}
	\centering
	\includegraphics[width=0.9\textwidth]{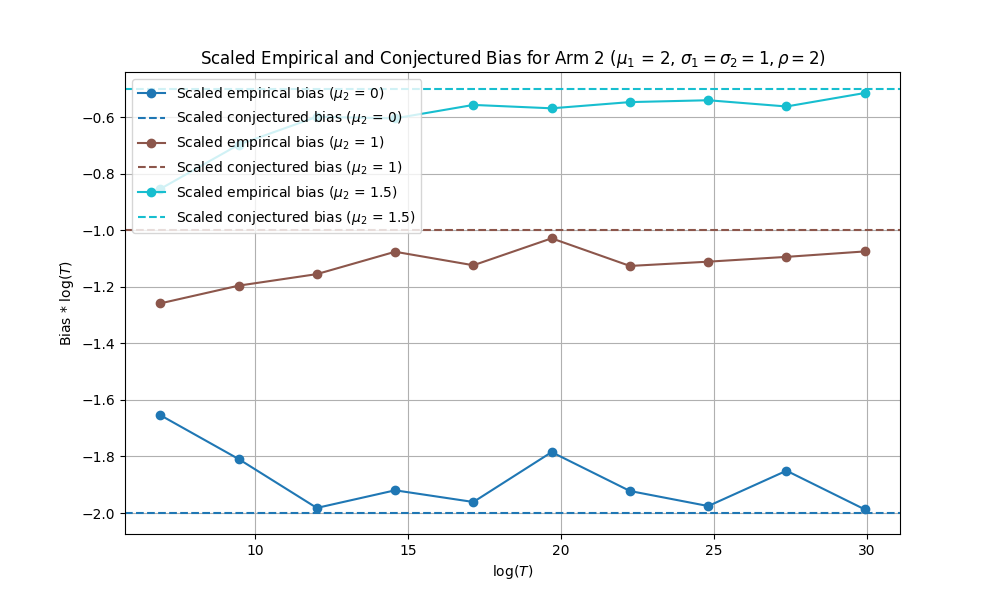}
	\caption{Scaled empirical bias of arm 2 under 10000 repetitions, $(\hat\mu_2(T)-\mu_2)\log T$, versus scaled (by $\log T$) conjectured bias of arm 2 in Conjecture~\ref{conjecture:sampling-bias-UCB}, $\sigma_2^2(\mu_1-\mu_2)$, for different horizon length $T$ We fix $\mu_1=2,\sigma_1=\sigma_2=1$ and $\rho=2$, and vary the value of $\mu_2$, represented by each curve.}\label{fig:scaled-bias-large-gap}
\end{figure}

To examine the moderate gap regime, we need the arm gap to be on the order of $\sqrt{\frac{\log T}{T}}$. In particular, we fix $\mu_2=0$ and $\sigma_1=\sigma_2=1$. Then for each $T$, let $n^{\star}_{1,T}={0.7T, 0.8T, 0.9T}$, and $\mu_1=\Delta_T=\sqrt{\frac{2\log T}{n^{\star}_{2,T}}}-\sqrt{\frac{2\log T}{n^{\star}_{1,T}}}=\sqrt{\frac{\theta\log T}{T}}$, where $\theta=\left(\sqrt{\frac{2T}{n^{\star}_{2,T}}}-\sqrt{\frac{2T}{n^{\star}_{1,T}}}\right)^2$, so that the fluid system of equations are satisfied. By Conjecture \ref{conjecture:sampling-bias-UCB}, the proper scaling of the empirical bias of arm 2 should be $\sqrt{T\log T}$, and the constant should be $-\frac{2\sqrt{1 + \lx}\sigma_2^2}{\sqrt{\rho}\left(1+(\lx)^{\frac{3}{2}})\right)}$. The comparison between the scaled empirical bias under 10000 repetitions and Conjecture \ref{conjecture:sampling-bias-UCB} is presented in Conjecture~\ref{fig:scaled-bias-moderate-gap}.
    
\begin{figure}
	\centering
	\includegraphics[width=0.8\textwidth]{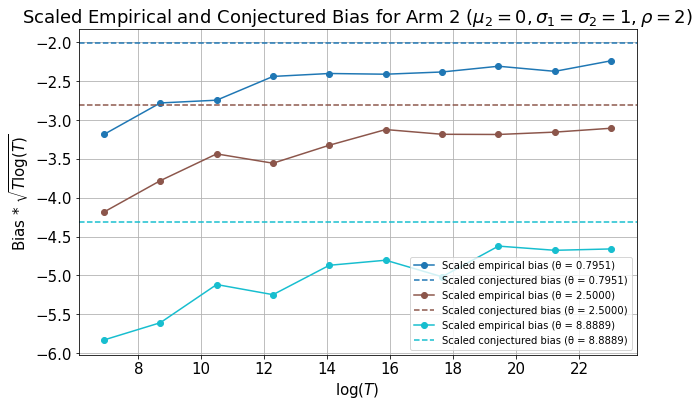}
	\caption{Scaled empirical bias of arm 2 under 10000 repetitions, $(\hat\mu_2(T)-\mu_2)\sqrt{T\log T}$, versus scaled (by $\sqrt{T\log T}$) conjectured bias of arm 2 in Conjecture~\ref{conjecture:sampling-bias-UCB} for different horizon length $T$ We fix $\mu_2=0, \sigma_1=\sigma_2=1$ and $\rho=2$, and vary the values of $\theta$ in $\mu_1=\sqrt{\frac{\theta\log T}{T}}$, represented by each curve.}\label{fig:scaled-bias-moderate-gap}
\end{figure}

\end{document}